\documentclass[letterpaper]{article}

\usepackage{aaai2026}
\usepackage{times}
\usepackage{helvet}    
\usepackage{courier}   
\usepackage[hyphens]{url}  
\usepackage{graphicx}  
\usepackage{braket}
\usepackage{caption}    
\usepackage[utf8]{inputenc} 
\usepackage{booktabs}       
\usepackage{amsfonts}       
\usepackage{nicefrac}       
\usepackage{microtype}      
\usepackage{xcolor}         
\usepackage{subcaption}
\usepackage{amsmath, amssymb, mathtools, amsthm}
\usepackage{tikz}
\usepackage{quantikz}
\usepackage{natbib}

\usepackage[capitalize,noabbrev]{cleveref}
\Crefname{equation}{Eq.}{Eq.}

\def\cU{\mathcal{U}}
\def\cX{\mathcal{X}}

\usepackage{amsmath}\allowdisplaybreaks
\usepackage{amsfonts,bm}
\usepackage{comment}
\usepackage{amssymb}
\usepackage{cleveref}
\usepackage{braket}

\providecommand{\<}{\langle}

\newcommand{\tr}{\mathrm{tr}}

\def\1{\bf{1}}

\def\vzero{{\bf{0}}}

\def\ve{{\bf{e}}}
\def\vf{{\bf{f}}}

\def\vs{{\bf{s}}}

\def\vu{{\bf{u}}}
\def\vv{{\bf{v}}}
\def\vw{{\bf{w}}}
\def\vx{{\bf{x}}}
\def\vy{{\bf{y}}}

\def\matA {{\bf A}}

\def\matD {{\bf D}}
\def\matE {{\bf E}}

\def\matG {{\bf G}}

\def\matI {{\bf I}}

\def\matR {{\bf R}}

\def\matU {{\bf U}}

\newcommand{\E}{\mathbb{E}}

\newcommand{\R}{\mathbb{R}}

\DeclareMathOperator*{\argmax}{arg\,max}
\DeclareMathOperator*{\argmin}{arg\,min}

\DeclareMathOperator{\diag}{diag}

\usepackage{amsthm}

\usepackage{etoolbox}

\makeatletter
\def\Ddots{\mathinner{\mkern1mu\raise\p@
\vbox{\kern7\p@\hbox{.}}\mkern2mu
\raise4\p@\hbox{.}\mkern2mu\raise7\p@\hbox{.}\mkern1mu}}
\makeatother

\makeatletter
\newcommand*{\rom}[1]{\expandafter\@slowromancap\romannumeral #1@}
\makeatother

\newtheorem{theorem}{Theorem}

\newtheorem{lemma}[theorem]{Lemma}

\theoremstyle{definition}
\newtheorem{definition}[theorem]{Definition}
\newtheorem{assumption}[theorem]{Assumption}
\theoremstyle{remark}
\newtheorem{remark}[theorem]{Remark}
\newtheorem{claim}[theorem]{Claim}

\def\vx {{{\bf x}}}
\def\vw {{{\bf w}}}
\def\vy {{{\bf y}}}

\def\C{{\bf C}}

\def\E{{\bf E}}

\def\N{{\bf N}}

\def\R{{\bf R}}

\def\w{{\bf w}}

\def\x{{\bf x}}

\def\Z{{\bf Z}}

\def\0{{\bf 0}}
\def\1{{\bf 1}}

\def\calA{{\mathcal A}}

\def\calD{{\mathcal D}}

\def\calF{{\mathcal F}}

\def\calL{{\mathcal L}}

\def\calO{{\mathcal O}}

\def\calW{{\mathcal W}}
\def\calX{{\mathcal X}}

\usepackage{algorithm}
\usepackage{algorithmic}
\usepackage{enumitem}
\usepackage{hhline}

\def \var #1{\text{\rm Var}\left[#1\right]}

\renewcommand{\>}{\rangle}
\renewcommand{\<}{\langle}

\usepackage{tikz}
\usepackage{quantikz}

\title{Quantum Non-Linear Bandit Optimization}
\author{
  Zakaria Shams Siam\textsuperscript{\rm 1},
  Chaowen Guan\textsuperscript{\rm 2},
  Chong Liu\textsuperscript{\rm 1}
  }
  \affiliations{
  \textsuperscript{\rm 1}University at Albany, State University of New York \quad\\
  \textsuperscript{\rm 2}University of Cincinnati\\
  \{zsiam,cliu24\}@albany.edu, guance@ucmail.uc.edu
  }

\begin{document}

\maketitle

\begin{abstract}
  We study non-linear bandit optimization where the learner maximizes a black-box function with zeroth order function oracle, which has been successfully applied in many critical applications such as drug discovery and materials design. Existing works have showed that with the aid of quantum computing, it is possible to break the classical $\Omega(\sqrt{T})$ regret lower bound and achieve the new $O(\mathrm{poly}\log T)$ upper bound. However, they usually assume that the objective function sits within the reproducing kernel Hilbert space and their algorithms suffer from the curse of dimensionality. In this paper, we propose the new Q-NLB-UCB algorithm which enjoys an \emph{input dimension-free} $O(\mathrm{poly}\log T)$ upper bound, making it applicable for high-dimensional tasks. At the heart of our algorithm design are quantum Monte Carlo mean estimator, parametric function approximation technique, and a new quantum non-linear regression oracle, which can be of independent interests in more quantum machine learning problems. Our algorithm is also validated for its efficiency compared with other quantum algorithms on both high-dimensional synthetic and real-world tasks.
\end{abstract}

\begin{links}
    \link{Code}{https://github.com/ZakSiam/Quantum-Non-Linear-Bandit-Optimization}
    \link{Extended version}{https://arxiv.org/abs/2503.03023}
\end{links}

\section{Introduction}

Non-linear bandit optimization, a.k.a., Gaussian process bandits, kernelized bandits, or Bayesian optimization, is a sequential decision making-based machine learning task that aims at solving a black-box optimization problem. Due to its black-box nature, it has been successfully applied in many important real-world applications where objective functions are difficult to define explicitly, e.g., hyperparameter tuning \citep{wu2020practical}, neural architecture search \citep{kandasamy2018neural}, drug discovery \citep{korovina2020chembo}, and materials science \citep{frazier2016bayesian}.

In drug screening, the objective is to identify a drug candidate from a large pool of compounds that exhibits the highest binding affinity to a specific biological target. Here each candidate is usually described by a feature vector $\vx$ and its binding affinity is a function of $\vx$, denoted as $f_0(\vx)$. Due to the highly complex chemical and biological reactions, the function $f_0$ is usually considered ``black-box'', which implies that it may be non-linear, non-convex, and even non-differentiable w.r.t. $\vx$. How to optimize $f_0$ then? The learner is allowed to sequentially query the function $f_0$. At each round $t$, the learner chooses to take a data point $\vx_t$ and observes its performance $y_t$, which is the outcome of a wet lab test. Obviously, a good algorithm can help the learner select promising data points as the experiment progresses and find the best candidate within fewest tests. 

In addition to its successful real-world applications, non-linear bandit optimization also enjoys solid theoretical guarantees. In literature, researchers usually define simple regret and cumulative regret to capture the convergence behavior of a certain algorithm, and many regret upper bounds have been established under different assumptions and with different kernels. All these positive theoretical results further contribute to more applications of non-linear bandit optimization. However, unfortunately an $\Omega(\sqrt{T})$ cumulative regret lower bound \citep{scarlett2017lower} cannot be further improved. What does that mean? It implies that given total $T$ rounds, no algorithm can incur  cumulative regret less than $\Omega(\sqrt{T})$ asymptotically. 

But can we do better? On the negative side, in classical (non-quantum) setting, the answer is ``no''. On the positive side, we have entered the quantum era where the power of quantum computing offers new hope for tackling this challenging optimization problem. \cite{WZL+23} first studied the multi-armed bandits and linear bandits and proved that new $O(\mathrm{poly} \log T)$ regret bound can be achieved with the aid of quantum computing. Later Q-GP-UCB \citep{DLV+23} and QMCKernelUCB \citep{hikimaquantum} studied the quantum Bayesian optimization, generalizing \cite{WZL+23} to non-convex and non-linear settings, still achieving the $O(\mathrm{poly} \log T)$ regret bound. However, \cite{DLV+23,hikimaquantum} both heavily rely on the reproducing kernel Hilbert space assumption, which suffers from the curse of dimensionality. The problem is that, in practice, many input data sit in high-dimensional spaces. Again in drug discovery, for example, the dimension of protein sequences usually ranges from thousands \citep{clarke2008properties} to even millions \citep{rahnenfuhrer2023statistical}. Then either $O(d^\frac{3}{2}_x (\log T)^\frac{3}{2})$ (linear kernel) or $O(T^\frac{3d_x}{2v+d_x})$ (Matérn kernel with $v$ being a kernel parameter) regret bound \citep{DLV+23} become vacuous when input dimension $d_x$ goes to millions, which further prevents their high-dimensional real-world applications. Therefore, can we design a new and efficient quantum non-linear bandit optimization algorithm that works well in high-dimensions?

In this paper, we answer this question affirmatively by proposing the Quantum Non-Linear Bandit with Upper Confidence Bound (Q-NLB-UCB) algorithm. The key design of Q-NLB-UCB relies on three techniques. First, Q-NLB-UCB runs in stages where in each stage the quantum oracle associated with the same action will be queried multiple times to achieve the quadratically improved sample complexity, guaranteed by quantum Monte Carlo mean estimator lemma \citep{Mon15}. Second, inspired by the success of parametric function approximation \citep{LW23}, we use a parametric function class to approximate and optimize the black-box objective function. All information queried is handled in the parameter space, so we are able to prove the first input \emph{dimension-free} regret bound for quantum non-linear bandit optimization. Finally, initialization of Q-NLB-UCB relies on a good estimated parameter $\hat{\vw}_0$ of the quantum non-linear regression problem. Its convergence to the optimal parameter $\vw^*$ enjoys a quadratic speed-up rate compared with classical non-linear regression, thanks to the quantum fast-forward technique \citep{AS19}.

\noindent\textbf{Contributions.} Our contributions are summarized as:

(1) We solve quantum non-linear bandit optimization with quantum computing and parametric function approximation, and propose the new Q-NLB-UCB algorithm. The design of algorithmic framework is generic and the choice of parametric function can be a linear function, quadratic function, or even multi-layer deep neural network, depending on tasks. 

(2) Different from existing works, Q-NLB-UCB does not suffer from the curse of dimensionality. We prove the first $O(d_w^2 \log^\frac{3}{2}(T)\log(d_w\log T))$ regret bound with $d_w$ being parameter complexity, which is also faster than the classical lower bound $\Omega(\sqrt{T})$ but \emph{independent} to input dimension $d_x$.

(3) Experiments on high-dimensional synthetic functions and real-world tasks show that Q-NLB-UCB outperforms compared algorithms in regrets and runtime.

\noindent \textbf{Technical Novelties.} The design of Q-NLB-UCB is a highly non-trivial task, involving tackling multiple technical challenges. Key technical novelties are highlighted as follows.

(1) The classical $O(1/\sqrt{T_0})$ regression bound \citep{Now09_12} with $T_0$ being number of samples converges too slow to work for Q-NLB-UCB. We introduce the quantum fast-forward technique \citep{AS19} to refine the analysis with Craig-Bernstein inequality \citep{Cra1933} to work with the quantum non-linear regression oracle, and achieve a \emph{quadratic} improvement in query complexity. This enhancement is highly non-trivial. As far as we know, \emph{no} existing classical or quantum methods can provide a comparable speed-up for the non-linear regression problem. Informally speaking, we prove that there exist such quantum regression solvers attaining this quadratic advantage, shedding light for a specific algorithm design in the future, which can be of independent interests in more quantum machine learning problems. While classical approaches \citep{diaconis2013spectral} offer some improvement, they are limited to constant-factor gains, rather than quadratic ones.

(2) In confidence analysis, when constructing the covariance matrix, we take the gradient w.r.t. the fixed initial parameter $\hat{\vw}_0$, rather than $\hat{\vw}_t$, which still makes rank-$1$ updates in each stage but avoids the tedious inductive argument in \cite{LW23}. After the first-order approximation using gradients, multiple parameters sitting in the same convex confidence region are used as bridges to apply convexity properties, ensuring the high-order terms are still well bounded.

\section{Related Work}\label{sec:rw}

\noindent \textbf{Classical Optimization.} Bayesian optimization \citep{frazier2018tutorial} is one of the most popular methods to solve global optimization. Based on the Gaussian process \citep{williams2006gaussian} or reproducing kernel Hilbert space \citep{chowdhury2017kernelized} assumption, Bayesian optimization runs in multiple rounds where in each round the learner takes an action suggested by an acquisition function. Common choices of acquisition function include expected improvement \citep{jones1998efficient}, knowledge gradient \citep{scott2011correlated}, upper confidence bound \citep{srinivas2010gaussian}, and Thompson sampling \citep{russo2018tutorial}. Without the classical Gaussian process assumption, \cite{snoek2015scalable,springenberg2016bayesian} used neural networks as the backbone surrogate models.

Besides Bayesian optimization, recent bandit works studied global optimization with neural network approximation \citep{zhou2020neural,zhang2021neural,dai2022sample} or generic parametric function approximation \citep{LW23}. In addition to bandit optimization, \cite{wang2019optimization} studied the global optimization of an unknown non-convex smooth function. When gradient information is available, \cite{allen2018katyusha,fang2018spider} also proposed algorithms to solve global optimization problem. However, our work is different from all of them since we focus on the quantum bandit optimization.

\noindent \textbf{Quantum Optimization.}
In recent years, there has been increasing interest in exploring quantum speed-up for optimization problems. This research direction began with quantum algorithms for linear and semi-definite programs \cite{BS17,AG18,CM20,KP20}, later extending to more general convex optimization \cite{vGGd20,CCLW20,he2022quantum,he2024quantum}. Recent advancements include quantum algorithms for slightly convex problems \cite{LZ22,CLW+23,zhang2024quantum}, escaping saddle points in non-convex landscapes \cite{ZLL21,CLL+22}, and identifying global minima in specific non-convex cases \cite{LSL23,LHLW23}. Alongside these algorithmic developments, quantum lower bounds have been established for both convex \cite{GKNS21tics,GKNS21nips} and non-convex optimization \cite{GZL22,ZL23}.

In a parallel line of research, stochastic quantum methods were proposed \cite{SZ23}, demonstrating the advantages of quantum stochastic first-order oracles for smooth objectives in low-dimensional settings. Most recently, there were efforts \cite{LGHL24} on investigating quantum speed-up for minimizing non-smooth, non-convex objectives, which represent the most general and fundamental function class. At the same time, \cite{zhang2024quantum} focused on studying quantum algorithms and lower bounds for finite-sum optimization, addressing both convex and non-convex cases.

\section{Preliminaries}\label{sec:pre}

\subsection{Problem Statement}

In this paper, we consider the non-linear bandit optimization problem: $\vx^* = \argmax_{\vx \in \mathcal{X}} f_0(\vx)$, where $f_0 : \mathcal{X} \rightarrow \mathcal{Y}$ is the unknown black-box objective function, which can be non-linear, non-convex, and not necessarily differentiable in $\vx$. $\mathcal{X} \subseteq \mathbf{R}^{d_x}$ is the function domain and $\mathcal{Y} \subseteq \mathbf{R}$ is the function range. To solve this problem, the learner has zeroth-order access to $f_0$ and the whole process runs in rounds. At each round $t=1,...,T$, after querying action $\vx_t$ the oracle will return a noisy function observation $y_t$. In classical (non-quantum) setting, after taking $\vx_t$, the function returns $y_t = f_0(\vx_t) + \eta_t$,
where $\eta_t$ is the zero-mean, independent, $\sigma$-sub-Gaussian noise. However, in (bounded value) quantum bandit setting, after taking the same action $\vx_t$ multiple times, the function oracle returns $y_t$ that satisfies $|y_t - f_0(\vx_t)| \leq \epsilon_t$, where $\epsilon_t$ is an error term. Either in classical or quantum setting, throughout $T$ rounds, we can always utilize the cumulative regret to evaluate the optimization process, $R_T = \sum_{t=1}^T f_0(\vx^*) - f_0(\vx_t)$, where $r_t = f_0(\vx^*)- f_0(\vx_t)$ is the instantaneous regret at round $t$. An algorithm $\mathcal{A}$ is said to be a no-regret algorithm if $\lim_{t \rightarrow \infty} R_T(\mathcal{A})/T \rightarrow 0$.

Since we are using a parametric function class to approximate the objective function, we use $\mathcal{W} \subseteq \mathbf{R}^{d_w}$ to denote the parameter class and its corresponding parametric function class is $\mathcal{F}=\{f_\vw:\mathcal{X} \rightarrow \mathcal{Y}|\vw \in \mathcal{W}\}$. Here we abuse the notation since $f_\vw(\vx)=f_\vx(\vw)$ and $f_\vx(\w)$ is used when $\vw$ is the variable of interest and $\vx$ is the parameter, and vice versa. 
Also, we use $\nabla f_\vx(\vw)$ and $\nabla^2 f_\vx(\vw)$ denote the gradient vector and Hessian matrix of function w.r.t. $\vw$. For a vector $\vx$, its $\ell_p$ norm is defined as $\|\vx\|_p = (\sum_{i=1}^d |\vx_i|^p)^{1/p}$. For a matrix $\matA$, its operator norm is denoted as $\|\matA\|_\mathrm{op}$. For a vector $\vx$ and a matrix $\matA$, let $\|\vx\|^2_\matA = \vx^\top \matA \vx$. Let $\ell(\cdot,\cdot): \mathbf{R}^2 \rightarrow \mathbf{R}$ denote a loss function, then the expected risk of a parameter $\vw$ is defined as $L(\vw) = \E_{(\vx,y) \sim \calD}[\ell(f_\vw(\vx), y)]$ with respect to distribution $\mathcal{D}$ and the empirical risk of $\vw$ is defined as $\hat{L}(\vw) = \frac{1}{n}\sum_{i=1}^n \ell(f_\vw(\x_i), y_i)$ with respect to $n$ data points. For readers' convenience, we use standard big $O$ notation to hide universal constants and use $\tilde{O}$ notation to further hide logarithmic factors.

\subsection{Quantum Computation}

\textbf{Quantum Basics.}
A quantum state can be seen as a vector $\vx = (x_1, x_2, \dots, x_m)^\top$ in Hilbert space $\mathbb{C}^m$ such that $\sum_i |x_i|^2 = 1$. We follow the Dirac bra/ket notation on quantum states, i.e.,  we denote the quantum state for $\vx$ by $|\vx \>$ and denote $\vx^{\dagger}$ by $\langle \vx|$ , where $\dagger$ means the Hermitian conjugation. Given a state $|\vx\> = \sum^m_{i=1} x_i |i\>$, we call $x_i$ the amplitude of the state $|i\>$. Given two quantum states $|\vx \rangle \in \mathbb{C}^m$ and $|\vy \rangle \in \mathbb{C}^m$, we denote their inner product by $\langle \vx|\vy \>:= \sum_i x^\dagger_i y_i$. Given $|\vx\> \in \mathbb{C}^n$ and $|\vy\> \in \mathbb{C}^m$, their tensor product is defined as $|\vx\> \otimes |\vy\> := (x_1 y_1, \cdots, x_1 y_m, \cdots, x_n y_1, \cdots, x_n y_m)^\top$. A quantum algorithm works by applying a sequence of unitary operators to an input quantum state. In many cases, the quantum algorithms may have access to input data via unitary operators called \emph{quantum oracles}. This operator can be accessed multiple times by a quantum algorithm. Hence, the \emph{quantum query complexity} of a quantum algorithm is defined as the number of a quantum oracle being used. See \cite{NC10} for detailed introductions to quantum computing.

\noindent
\textbf{Quantum Noisy Function Oracle.} Our quantum non-linear bandit optimization setting follows that of the quantum multi-armed bandits in \cite{WZL+23}. In the quantum realm, each input $\vx$ is associated with a quantum sampling oracle. This oracle follows quantum sampling oracle (see appendix) and encodes the distribution of the corresponding noisy function value. More formally, let $Y_\vx$ be the random variable of the noisy function value with input $\vx$, and let $\Omega_\vx$ be the finite sample space of this distribution. Then the sampling oracle for the noisy function value with input $\vx$ is defined as: 
\begin{equation}\label{eq:noisy_func_value}
    \mathcal{O}_\vx: |\vzero\> \rightarrow \sum_{y \in \Omega_\vx} \sqrt{\Pr[Y_\vx = y]}\ |y \> \otimes |\psi_y\>,
\end{equation}
where $|\psi_y\>$ is an arbitrary quantum state for each $y$. More detailed discussions on the justification of quantum oracles' feasibility and the relationships between them and their classical counterparts can be found in appendix.

\noindent
\textbf{Quantum Mean Estimation.} For estimating the mean of an unknown distribution, we will use the following quantum Monte Carlo mean estimator as in \cite{WZL+23,WGAW23,DLV+23}:
\begin{lemma}[Quantum Monte Carlo mean estimator \cite{Mon15}]\label{lem:qme}
Given the access to a quantum sampling oracle $\mathcal{O}_Y$ (and its inverse $\mathcal{O}_Y^\dagger$) that encodes the distribution of a random variable $Y$, as defined in Eq. \eqref{eq:noisy_func_value}.

(1) \textbf{Bounded value}: If the value of $Y$ is taken from the interval $[0,1]$, then there exists a constant $C_1 > 1$ a quantum algorithm $\mathsf{QME}_1(\mathcal{O}_Y, \epsilon, \delta)$ which returns an estimate $\hat{y}$ such that with probability at least $1 - \delta$, $|\hat{y} - \E[Y]| \leq \epsilon$,
using at most $\frac{C_1}{\epsilon} \log(1/ \delta)$ queries to $\mathcal{O}_Y$ and its inverse.

(2) \textbf{Bounded variance}: If $\mathrm{Var}[Y] \leq \sigma^2$, then for $\epsilon < 4 \sigma$, there is a constant $C_2 > 1$ and a quantum algorithm $\mathsf{QME}_2(\mathcal{O}_Y, \epsilon, \delta)$ which returns an estimate $\hat{y}$ such that with probability at least $1 - \delta$, $|\hat{y} - \E[Y]| \leq \epsilon$, using at most $\frac{C_2 \sigma}{\epsilon} \log^{3/2}(8\sigma/\epsilon)\log(\log (8\sigma/\epsilon))\log(1/\delta)$ queries to $\mathcal{O}_Y$ and its inverse.
\end{lemma}
As briefly discussed in \cite{WZL+23}, when aiming for a mean estimation error of $\epsilon$, the QME algorithm achieves a quadratic reduction in query complexity compared to the classical one, which is crucial for the quantum speed-up in \cite{WZL+23,WGAW23,DLV+23}.

\subsection{Assumptions}
\begin{assumption}[Realizable parametric function class]\label{asm:real}
There exists an optimal $\vw^* \in \mathcal{W}$ such that $f_0 = f_{\vw^*}$. Also, w.l.o.g., $\mathcal{W} \subseteq [0,1]^{d_w}$.
\end{assumption} 
This assumption is commonly used in bandits \citep{foster2020beyond,simchi2022bypassing} and reinforcement learning \citep{zhan2022offline,zanette2023realizability}. In Bayesian optimization \citep{srinivas2010gaussian,chowdhury2017kernelized}, the RKHS assumption essentially assumes that the objective function is realizable within a certain RKHS function class. The realizable assumption allows one not to handle the function in misspecified settings \cite{bogunovic2021misspecified,liu2023no}, which is beyond the scope of this paper. The second assumption is on the structure of parametric function $f_\vw$.
\begin{assumption}[Bounded, differentiable, and smooth function]\label{asm:bounded}
There exist constants $C_f, C_g, C_h > 0$ such that $\forall \vx \in \calX, \forall \vw \in \calW$, it holds that $|f_\vx(\vw)| \leq C_f, \| \nabla f_\vx(\vw)\|_2 \leq C_g, \|\nabla^2 f_\vx(\vw) \|_{\mathrm{op}} \leq C_h$. 
\end{assumption}
This is a common assumption used in many non‑convex optimization works \cite{kohler2017sub,li2023convex}. Note it only puts mild structure conditions on the smoothness of parametric function $f_\vw$ w.r.t. its parameter $\vw$, rather than input $\vx$, and the objective function $f_0$ can still be a black-box function of $\vx$.
The last assumption is on the expected loss function over uniform distribution $\mathcal{U}$.
\begin{assumption}[Geometric conditions of loss function \citep{LW23}]\label{asm:loss}
$L(\vw)=\E_{\vx \sim \mathcal{U}}(f_\vx(\vw)-f_\vx(\vw^*))^2$ satisfies $(\tau, \gamma)$-growth condition or $\mu$-local strong convexity at $\vw^*$, i.e., $\forall \vw \in \mathcal{W}$,
\begin{align*}
\min\left\{\frac{\mu}{2}\|\vw-\vw^*\|_2^2,\frac{\tau}{2}\|\vw-\vw^*\|_2^\gamma \right\} \leq L(\vw)-L(\vw^*),
\end{align*}
for constants $\mu,\tau >0,\mu < d_w$ and $0<\gamma<2$. Also, $L(\vw)$ satisfies a $c$-local self-concordance assumption at $\vw^*$.
\end{assumption}
This assumption is needed for technical reasons in analyzing quantum regression oracle. Note the loss function $L(\vw)$ can be a highly non-convex function since it only assumes \emph{local} strong convexity in the neighboring region of $\vw^*$, strictly weaker than the global strong convexity, and growth condition when $\vw$ is away from $\vw^*$. Careful readers are referred to Figure 1 in \cite{LW23} for a non-convex function example satisfying this assumption.

\section{Q-NLB-UCB Algorithm}\label{sec:algorithm}

In this section, we show full details of the Q-NLB-UCB algorithm (Algorithm \ref{alg:q_go_ucb}).
First, in Step 1, we take a subroutine Algorithm \ref{alg:qra} to query the quantum regression oracle $\mathsf{QNLRO}$ for $T_0$ times, which aims at solving the following non-linear regression problem to get an estimated parameter 
$
\hat{\vw}_0 \leftarrow \argmin_{\vw \in \mathcal{W}} \sum_{j=1}^{T_0} (f_0(\vx_j) - y_j)^2.
$
Our goal is to make sure that $\hat{\vw}_0$ satisfies 
\begin{align}
\|\hat{\vw}_0 - \vw^*\|_2 \leq \frac{C_0}{T_0},\label{eq:w0}
\end{align}
where $C_0$ denotes a constant. Careful readers may have noticed that in the classical (non-quantum) regime, the best upper bound is only $\|\hat{\vw}_0 - \vw^*\|_2 \leq O(1/\sqrt{T_0})$, which can be obtained using the small variance property of squared losses near optimal solution and applying the Craig-Bernstein (CB) inequality \citep{Cra1933}. How can we achieve the quadratic improvement in quantum case? In short, we work with the quantum fast-forward technique \citep{AS19} and refine the analysis with the CB inequality. This gives the desirable convergence rate, but the parameter vector  $\hat{\vw}_0$ is returned in the form of a quantum state. To use it in classical state later, necessary techniques for retrieving the classical information of all the entries in $\hat{\vw}_0$ will be employed. Specifically, we use \emph{non-destructive amplitude estimation} ($\mathsf{NDAE}$) \cite{RF23}. More details on techniques and analyses are given in the next section and full proofs are shown in appendix.

From Step 2 to Step 12, the algorithm runs in $m$ stages where multiple rounds are conducted in each stage. Why doesn't it run in simply $T$ rounds like classical bandit optimization? This design is due to the quantum Monte Carlo mean estimation (Lemma \ref{lem:qme}) where the same action needs to be taken multiple times.
And we set $\epsilon_s = \|\nabla f_{\vx_s}(\hat{\vw}_0)\|_{\Sigma^{-1}_s}$ and $m = d_w \log \left( \frac{C^2_g T^2}{d_w \lambda}+ 1\right)$ (Lemma \ref{lem:rounds}) to ensure the total number of rounds is $T$. In addition to using $\mathsf{QME}_1$ in Algorithm \ref{alg:q_go_ucb}, our algorithm also works with $\mathsf{QME}_2$ which is the bounded variance case in Lemma \ref{lem:qme}. The analysis will be similar to that of $\mathsf{QME}_1$.

\begin{algorithm}[t]
\caption{Q-NLB-UCB}
\label{alg:q_go_ucb}
{\bf Input:}
Objective function $f_0$, initial covariance matrix $\Sigma_0 =\lambda \matI$, quantum non-linear regression oracle $\mathsf{QNLRO}$, regularization weight $\lambda$, confidence sequence $\beta_s$, constant $C_1$.
\begin{algorithmic}[1]
\STATE $\hat{\vw}_0 \leftarrow \mathsf{QNLRO}(f_0, T_0,\delta/4)$
\FOR{each stage $s = 1,2,\cdots$}
\STATE Update $\Sigma_s$ by Eq. \eqref{eq:sigma_s}.
\STATE Update $\hat{\vw}_s$ by Eq. \eqref{eq:opt_inner}.
\STATE Update $\mathrm{Ball}_s$ by Eq. \eqref{eq:ball}.
\STATE Select $\vx_s=\argmax_{\vx \in \cX} \max_{\vw \in \mathrm{Ball}_s} f_\vx(\vw)$.
\STATE Update $\epsilon_s = \|\nabla f_{\vx_s}(\hat{\vw}_0)\|_{\Sigma^{-1}_{s}}$.
\FOR{the next $\frac{C_1}{\epsilon_s} \log \frac{m}{\delta}$ rounds}
\STATE Take actions $\vx_s$ and run $\mathsf{QME}_1(O_{\vx_s}, \epsilon_s, \delta/m)$.
\STATE Obtain $y_s$ as an estimation of $f_0(\vx_s)$.
\ENDFOR
\ENDFOR                                
\end{algorithmic}
{\bf Output:} $\hat{\vx} \sim \cU (\{\vx_1, \cdots, \vx_T\})$.
\end{algorithm}

\begin{algorithm}[t]
\caption{Quantum Non-Linear Regression Oracle ($\mathsf{QNLRO}$)}
\label{alg:qra}
{\bf Input:} Objective function $f_0$, time $T_0$, failure parameter $\delta \in [0, 1/2]$, quantum regression oracle $\mathsf{Oracle}$.
\begin{algorithmic}[1]
\STATE $|\hat{\vw}_0\> \leftarrow \mathsf{Oracle}(f_0, T_0, \delta/4)$  
\FOR{$i = 1, \cdots, d_w$}
\STATE set projector $P_i = |i\> \<i|$
\STATE obtain $\tilde{a}_i \leftarrow \mathsf{NDAE}(|\hat{\vw}_0\>, P_i, \frac{1}{d_w \cdot T_0^2}, \frac{\delta}{4 d_\vw})$
\ENDFOR
\end{algorithmic}
{\bf Output:} $\hat{\vw}_0 = (\tilde{a}_1, \cdots, \tilde{a}_{d_w})$.
\end{algorithm}

Specifically in each stage $s=1,...,m$, in Step 3, we construct the covariance matrix $\Sigma_s$ by
\begin{equation}
\Sigma_s =\Sigma_{s-1} + \frac{1}{\epsilon^2_{s-1}} \nabla f_{\vx_{s-1}}(\hat{\vw}_0)\nabla f_{\vx_{s-1}}(\hat{\vw}_0)^\top,\label{eq:sigma_s}
\end{equation}
where the $1/\epsilon^2_{s-1}$ is the weight assigned to query in each stage. Note here $\nabla f_{\vx_{s-1}}(\hat{\vw}_0)$ is the gradient of the parametric function $f$ taken w.r.t. $\hat{\vw}_0$, which can be easily obtained, and the objective function $f_0$ can still be a black-box function without any derivative information. Different from \cite{LW23}, $\nabla f_{\vx_{s-1}}(\hat{\vw}_0)$ is not taken w.r.t. the fixed $\hat{\vw}_{s-1}$ to save the tedious inductive argument in it while still doing the rank-$1$ updates since the action $\vx_{s-1}$ changes over stages. And the rank-$1$ updates are needed because the algorithm can save all historical information and add only one new matrix at each stage, according to Eq. \eqref{eq:sigma_s}. Then we define the following regression problem to estimate $\hat{\vw}_s$:
\begin{align}\label{eq:opt_inner}
&\hat{\vw}_s = \argmin_{\vw}\frac{\lambda}{2} \|\vw - \hat{\vw}_0\|_2^2 \nonumber\\
&\quad + \frac{1}{2} \sum_{i=0}^{s-1} \frac{1}{\epsilon^2_i}\left((\vw-\hat{\vw}_0)^\top \nabla f_{\vx_i}(\hat{\vw}_0) + f_{\vx_i}(\hat{\vw}_0) - y_i\right)^2.
\end{align}
Note in the first term $\|\vw - \hat{\vw}_0\|_2^2$, the regression center is set to be $\hat{\vw}_0$ so that we can take advantage of Eq. \eqref{eq:w0} to reach a much faster convergence rate than constant as in QLinUCB \citep{WZL+23}. The design of the second term is using the first order Taylor expansion of parametric function $f_{\vx_i}$ to approximate the noisy observation $y_i$. Solution to optimization problem in Eq. \eqref{eq:opt_inner}, $\hat{\vw}_s$, further serves as the center of the parameter uncertainty region $\mathrm{Ball}_s$, defined as
\begin{align}\label{eq:ball}
\mathrm{Ball}_s = \{\vw \in \mathbb{R}^d: \|\vw - \hat{\vw}_s\|^2_{\Sigma_{s}} \leq \beta_s\}.
\end{align}
The key design of $\mathrm{Ball}_s$ is to contain the optimal parameter $\vw^*$ in each stage $s$ w.h.p., so that Q-NLB-UCB can keep track of $\vw^*$ at all times. The radius parameter $\beta_s$ plays an important role in this design and later in Lemma \ref{lem:beta} our confidence analysis shows that it suffices to choose $\beta_s$ as
\begin{align}\label{eq:beta}
\beta_s = 3d_w s + \frac{3\lambda C_0^2}{T_0^2} + \frac{3C^2_h C_0^2 s T^2}{4 T_0^4}.
\end{align}
In Step 6, the choice of $\vx_s$ is generated by solving a cross optimization problem defined in both $\mathcal{X}$ and $\mathrm{Ball}_s$. While the exact solution of this step is hard to find, in practice one can use gradient ascent as a surrogate solution, and it works well in our experiments.
The final output $\hat{\vx}$ is uniformly sampled from all historical actions $\vx_1,...,\vx_T$ since $f^* - \E f(\hat{\vx}) \leq R_T/T$, i.e., one can easily obtain the theoretical guarantee of $\hat{\vx}$ using cumulative regret bound with uniform sampling. But in practice, one can also select $\vx_T$ as the output.

\section{Theoretical Analysis}\label{sec:theory}

In this section, we provide theoretical analysis for Q-NLB-UCB. First, we analyze the quantum regression oracle that outputs $\hat{\vw}_0$ to start the algorithm, then we provide the regret analysis to prove its input dimension-free $O(\mathrm{poly}\log T)$ regret bound, supported by the confidence analysis. Full proofs and time complexity analysis are deferred to appendix.

\subsection{Quantum Non-Linear Regression Oracle}\label{sec:theory_oracle}

Here we show the existence of a quantum non-linear regression oracle which outputs the estimated parameter $\hat{\vw}_0$ that satisfies Eq. \eqref{eq:w0}. This algorithm involves two primary procedures: (1) obtaining a quantum state that encodes $\hat{\vw}_0$; (2) retrieving the classical information of all entries in $\hat{\vw}_0$. Roughly speaking, this result can be achieved by ``quantizing'' the proof of Theorem 5.2 in \cite{LW23} with quantum fast-forward technique \citep{AS19}, shown below.

\begin{lemma}[Adapted from Theorem 5.2 in \cite{LW23}]\label{lm:parameter_dist}
Suppose Assumptions \ref{asm:real}, \ref{asm:bounded}, and \ref{asm:loss} hold. There is an absolute value $C$ such that after $T_0$ iterations in step 1 of \Cref{alg:q_go_ucb} where $T_0$ satisfies $T_0 \geq C d_\vw \iota \max\left\{ \frac{\mu^{\gamma/(2-\gamma)}}{\tau^{2/(2-\gamma)}}, \frac{2C^2_g}{\mu c^2} \right\}$, with probability $1-\delta/2$, the quantum regression oracle returns an estimate $\hat{\vw}_0$ that satisfies $\|\hat{\vw}_0 - \vw^*\|_2 \leq \sqrt{\frac{C d_\vw \iota}{T_0}}$, where $\iota$ is the logarithmic term depending on $T_0, C_h, 1/\delta$.
\end{lemma}

Note that this lemma still gives $O(1/\sqrt{T_0})$ asymptotical rate after $T_0$ iterations, the same as in classical setting. To attain the desired faster $O(1/T_0)$ rate, we apply the quantum fast-forward technique introduced in \cite{AS19}. An informal description below follows that in \cite{AGJK20}.

\begin{lemma}[Informal statement of quantum fast-forward \citep{AGJK20}]\label{thm:QFF_informal}
    Let $\epsilon \in (0,1), s \in [0,1]$ and $t \in \N$. Given a reversible Markov chain determined by a matrix $\matD$, there is a quantum algorithm with $O(\sqrt{t \log(1/\epsilon)})$ quantum walk steps that takes input $|\overline{0}\>|\psi\> \in \mathrm{span}\{|\overline{0}\>|x\> : x \in X\}$, and outputs a state that is $\epsilon$-close to a state of the form $|0\>^{\otimes a} |\overline{0}\> D^t|\psi\> + |\Gamma\>$
    where $a=O(\log(t\log(1/\epsilon)))$, $|\overline{0}\> $ is some fixed reference state, and $|\Gamma\>$ is some garbage state that has no support on states containing $|0\>^{\otimes a} |\overline{0}\> $ in the first two registers.
\end{lemma}

Consequently, we can summarize the goal of Step 1 in \Cref{alg:q_go_ucb} with the following theorem. Full proofs are deferred to appendix.

\begin{theorem}\label{thm:phaseI}
    Suppose Assumptions \ref{asm:real}, \ref{asm:bounded}, and \ref{asm:loss} hold. There is an absolute value $C$ such that after $\tilde{O}(T_0)$ iterations in step 1 of \Cref{alg:qra} where $T_0$ satisfies $T_0^2 \geq C d_\vw \iota \max\left\{ \frac{\mu^{\gamma/(2-\gamma)}}{\tau^{2/(2-\gamma)}}, \frac{2C^2_g}{\mu c^2} \right\}$
    with probability $1-\delta/2$, the quantum regression oracle returns a quantum state that encodes an estimate $\hat{\vw}_0$ that satisfies $\|\hat{\vw}_0 - \vw^*\|_2 \leq \frac{\sqrt{C d_\vw \iota}}{T_0},$
    where $\iota$ is the logarithmic term depending on $T_0, C_h, 1/\delta$.
\end{theorem}

So far we have proved the convergence of information provided by $|\hat{\vw}_0\>$. This means that after applying the quantum regression oracle, we will be returned with a quantum state $|\hat{\vw}_0 \>$ which encodes the parameter vector $\hat{\vw}_0$ that meets the convergence guarantee. Next for the classical usage of parameter vector $\hat{\vw}_0$, we need to retrieve them from the quantum state $|\hat{\vw}_0\>$. A standard approach will be employing quantum state tomography. But this will be an overkill because in general quantum state tomography considers the cases where the $\log N$-qubit quantum state input to be in $\C^N$. Recall that with loss of generality, our work focuses on $|\hat{\vw}_0 \> \in [0,1]^{d_w}$. Hence, we can employ a more efficient tool, named \textit{quantum amplitude estimation}, which outputs an estimate of $\<\psi|P|\psi\>$ upon input a quantum state $|\psi\>$ and a projector $P$. However, the execution of straight quantum amplitude estimation algorithm will cause the input state $|\psi\>$ to collapse, which means $O(\mathrm{poly}(d_w, \epsilon, \delta))$ many copies of the input quantum states are needed for gaining the classical information of all the entries in $|\psi\>$. In our problem, this preparation of multiple copies of the state $|\hat{\vw}_0 \>$ will require queries to the quantum sampling oracles $\calO_\vx$, thus dramatically increasing the cumulative regret, which is undesirable. That means that our problem lies in scenarios where $|\psi\>$ is extremely expensive to prepare. To avoid this, we turn to \textit{non-destructive amplitude estimation} which can return an estimation of $\<\psi|P|\psi\>$ and also give the copy of $|\psi\>$ back. There are multiple existing works in this field. \cite{RF23} listed multiple of them and we pick one among them. An informal description is stated below.

\begin{theorem}[(Informal) Non-destructive amplitude estimation in \cite{RF23}]\label{thm:ndae_informal}
    Given one copy of a quantum state $|\psi\>$, a projector $P$, and $\epsilon, \delta \in [0,1/2]$. Let $a = \<\psi|P|\psi\>$. Then there is an algorithm $\mathsf{NDAE}(|\psi\>, P, \epsilon, \delta)$ that with probability at least $1-\delta$, outputs an estimate $\tilde{a}$ and a copy of $|\psi\>$ such that $|\tilde{a} - a| \leq \epsilon$.
\end{theorem}

Let \textsf{Oracle} denote a quantum algorithm that solves the non-linear regression problem and outputs a quantum state encoding the solution that satisfies the convergence bound provided by \Cref{thm:phaseI}. Combining \Cref{thm:phaseI} and \Cref{thm:ndae_informal}, we obtain the following result for the quantum non-linear regression oracle ($\mathsf{QNLRO}$, \Cref{alg:qra}).

\begin{theorem}\label{thm:qro}
    Suppose Assumptions \ref{asm:real}, \ref{asm:bounded}, and \ref{asm:loss} hold. Then \Cref{alg:qra} returns with probability $1-\delta /2$, a classical vector of an estimate $\hat{\vw}_0$ that satisfies $\|\hat{\vw}_0 - \vw^*\|_2 \leq \frac{\sqrt{C d_\vw \iota}}{T_0}$,
    where $\iota$ is the logarithmic term depending on $T_0, C_h, 1/\delta$ and $T_0$ satisfies
    $T_0^2 \geq C d_\vw \iota \max\left\{ \frac{\mu^{\gamma/(2-\gamma)}}{\tau^{2/(2-\gamma)}}, \frac{2C^2_g}{\mu c^2} \right\}$.
\end{theorem}

The above theorem implies that it is sufficient to use only one copy of $|\hat{\vw}_0\> $ to extract the classical information of all of it's entries, given the fact that all the entries belong to $[0,1]$. More concretely, \Cref{alg:qra} first calls a quantum algorithm $\mathsf{Oracle}$ to obtain a quantum state that has the solution parameter vector. Then it uses $\mathsf{NDAE}$ with projectors that project onto each computational basis, i.e. $\<\hat{\vw}_0| i\>\<i|\hat{\vw}_0\>$, to retrieve each entry, respectively. Note that each time $\mathsf{NDAE}$ returns not only an estimate of an entry but also a copy of the original $|\hat{\vw}_0\> $. Additionally, since $\mathsf{NDAE}$ doesn't query the sampling oracle $\calO_\vx$, its execution incurs no extra cumulative regret.

\subsection{Regret Analysis}\label{sec:regret}

\begin{figure*}[t]
    \centering
    \begin{minipage}{0.24\linewidth}\centering
		\includegraphics[width=\textwidth]{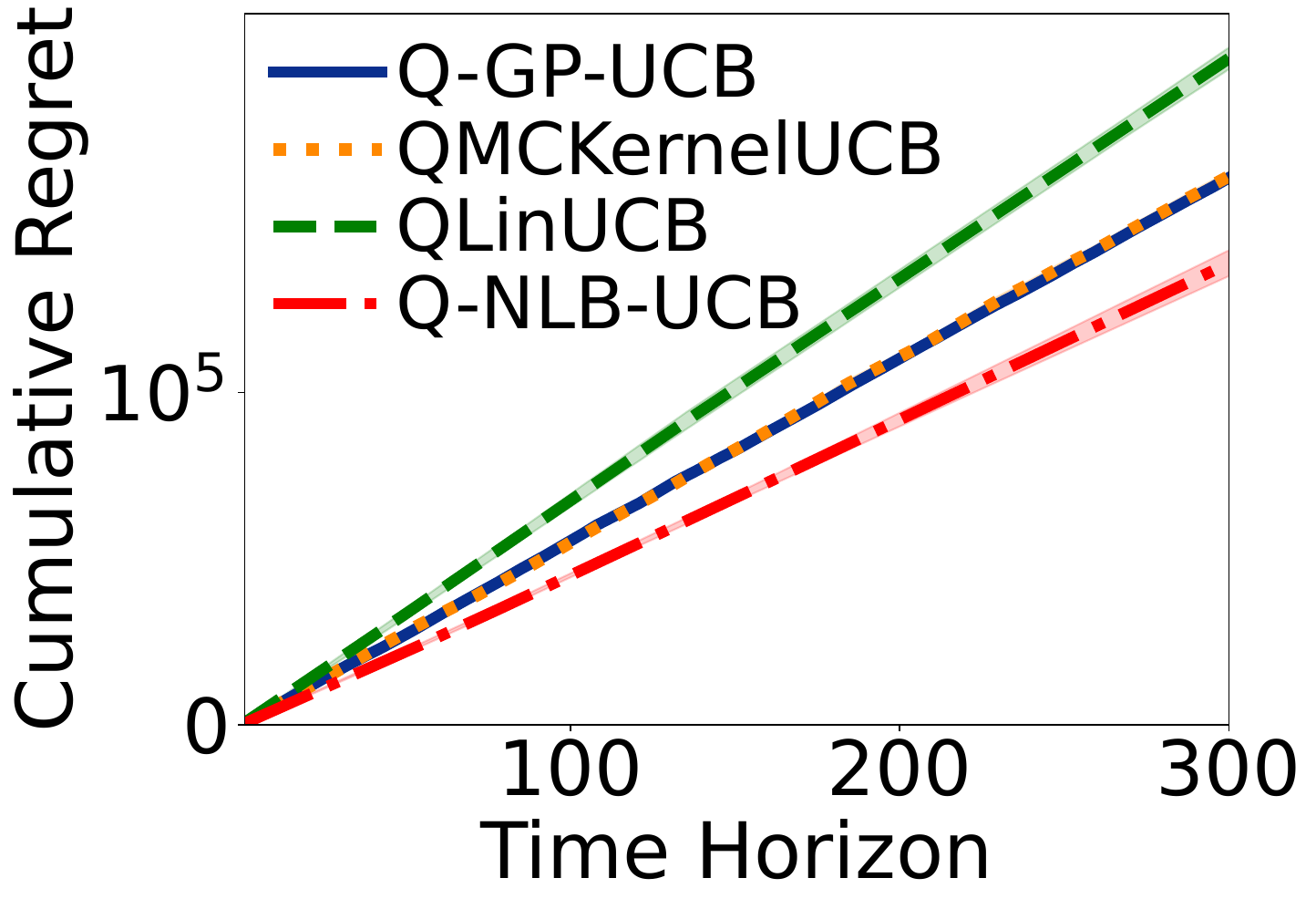}\\[0.05ex]
    \small (a) 30D Rastrigin
	\end{minipage}
    \begin{minipage}{0.24\linewidth}\centering
		\includegraphics[width=\textwidth]{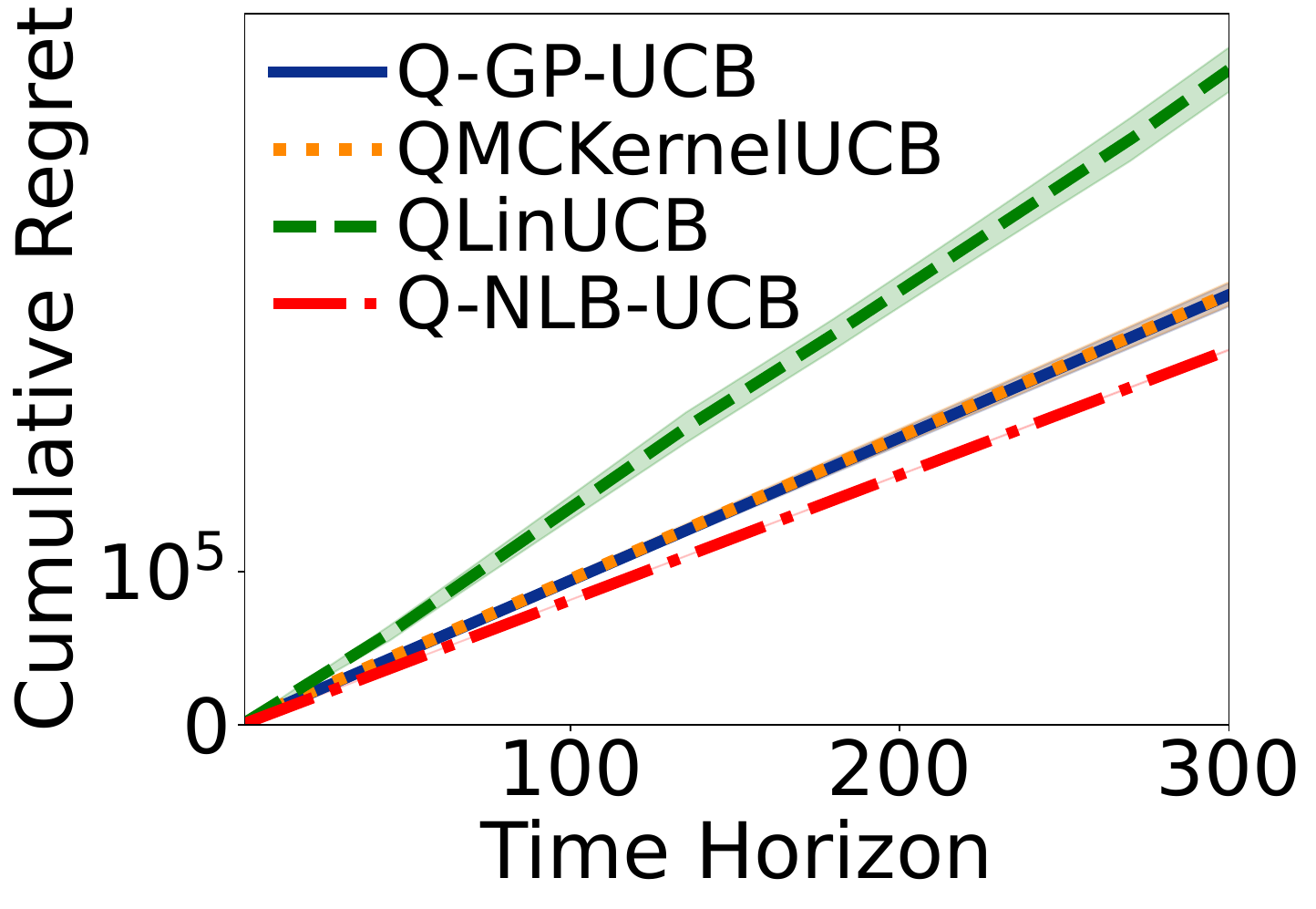}\\[0.05ex]
    \small (b) 30D Styblinski-Tang
	\end{minipage}
    \begin{minipage}{0.23\linewidth}\centering
		\includegraphics[width=\textwidth]{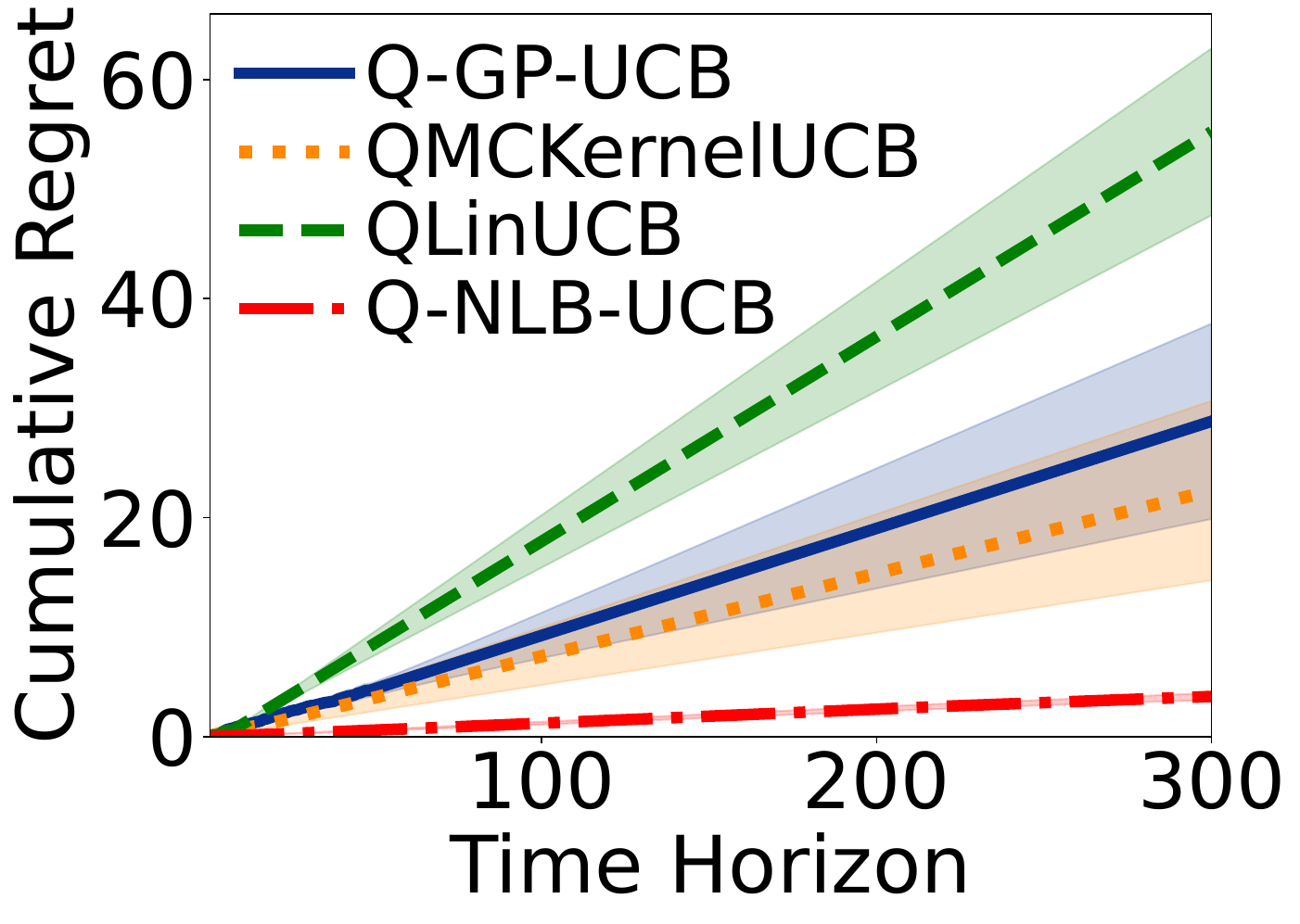}\\[0.05ex]
    \small (c) AutoML (Cancer)
	\end{minipage}
    \begin{minipage}{0.23\linewidth}\centering
		 \includegraphics[width=\textwidth]{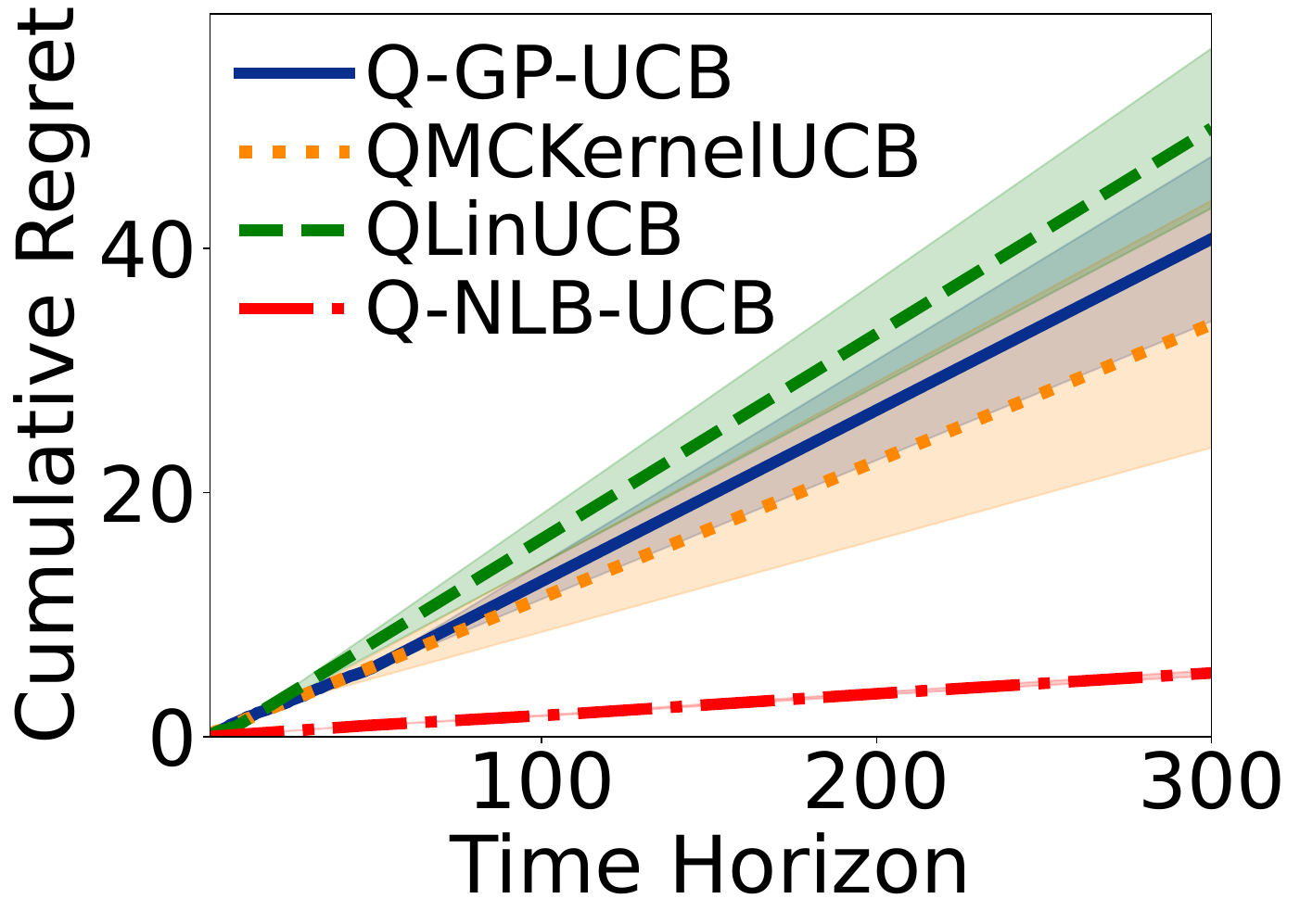}\\[0.05ex]
    \small (d) AutoML (Diabetes)
	\end{minipage}
    \caption{Cumulative regrets (the lower the better) of all compared quantum bandit algorithms.}
    \label{fig:cr}
\end{figure*}

Now we present the cumulative regret bound of our proposed Q-NLB-UCB algorithm.
\begin{theorem}[Cumulative regret bound of Q-NLB-UCB]\label{thm:cr}
Suppose Assumptions \ref{asm:real}, \ref{asm:bounded}, and \ref{asm:loss} hold. There is an absolute value $C$ such that after $\tilde{O}(T_0)$ iterations in Step 1 of \Cref{alg:q_go_ucb} where $T_0$ satisfies $T^2_0 \geq C d_\vw \iota \max\left\{ \frac{\mu^{\gamma/(2-\gamma)}}{\tau^{2/(2-\gamma)}}, \frac{2C^2_g}{\mu c^2} \right\}$ with $\iota$ denoting a logarithmic term depending on $T_0, C_h, 1/\delta$. Then Algorithm \ref{alg:q_go_ucb} with parameters $T_0=\sqrt{T}, \lambda=T$ satisfies that with probability at least $1-\delta$, $R_T = O\big(d_w^2 \log^\frac{3}{2}(T) \log (d_w \log (T))\big)$.
\end{theorem}
\begin{remark}
Note $d_w$ is the parameter complexity of $f_\vw$, which is \emph{not} necessarily related to the input dimension $d_x$ of $f_0$, therefore, our regret bound is {input dimension-free}. In practice, $d_w$ can be smaller or larger than $d_x$ since $d_w$ solely depends on the users' choice of parametric functions, which can be linear or quadratic functions, or even deep neural networks. When $f_\vw$ is chosen to be a linear function, our algorithm reduces back to the linear bandits as in QLinUCB \citep{WZL+23}. Compared with Q-GP-UCB \citep{DLV+23} and QMCKernelUCB \citep{hikimaquantum}, our algorithm takes a different technical route to successfully avoid the curse of dimensionality limitation. Moreover, our bound is at the $\log^\frac{3}{2}(T)\log\log T$ rate, which is also faster than classical lower bound $\Omega(\sqrt{T})$, showing the power of quantum computing.
\end{remark}
\begin{remark}
The choices of $T_0=\sqrt{T},\lambda=T$ require careful analyses among quantum regression oracle, regret analysis, and confidence analysis. $T_0$ cannot be chosen too large to enforce a large $T$ and it cannot be chosen too small to break the property of $\hat{\vw}_0$ in Eq. \eqref{eq:w0}. The choice of $\lambda$ is obtained by balancing between different terms in regret analysis to ensure a $\mathrm{poly} \log T$-style bound. 
\end{remark}

\noindent \textbf{Proof Sketch.} The proof starts from the instantaneous regret in a single round in one stage. Since we are dealing with non-linear bandit optimization where the objective function is not necessarily linear, we use Taylor's theorem to expand the objective function into first order terms and high-order terms. The first order terms are handled like linear bandits \citep{WZL+23}, but the remaining high-order terms are bounded creatively using the convex property of parameter uncertainty region $\mathrm{Ball}_s$ multiple times. After we obtain the upper bound for a single round in one stage, we multiply it by number of rounds in one stage and obtain the bound for one stage. Later, we prove the following lemma to show the total number of stages in Q-NLB-UCB to make sure the total number of rounds is $T$.
\begin{lemma}[Number of stages]\label{lem:rounds}
The Algorithm \ref{alg:q_go_ucb} runs at most $m = d_w \log\left( \frac{C_g^2 T^2}{d_w \lambda} + 1 \right)$ stages.
\end{lemma}
Note again in this lemma, $d_w$ is the dimension of parameters, rather than $d_x$. This is because we are using $\nabla f_{\vx_i}(\hat{\vw}_0)$ as the feature vector and its dimension is $d_w$. After proving Lemma \ref{lem:rounds}, we take the summation of upper bounds and reach the upper bound of cumulative regret.

\subsection{Confidence Analysis}

The previous regret builds upon the successful construction of the confidence ball $\mathrm{Ball}_s$ for each stage $s$, which is summarized in the following confidence analysis lemma.
\begin{lemma}[Confidence bound of Q-NLB-UCB]\label{lem:beta}
Suppose Assumptions \ref{asm:real}, \ref{asm:bounded}, and \ref{asm:loss} hold and $\beta_s$ is chosen as Eq. \eqref{eq:beta}. Then with parameters $T_0=\sqrt{T},\lambda=T$ in each stage $s$ in Algorithm \ref{alg:q_go_ucb}, the optimal parameter $\vw^*$ is trapped in confidence ball $\mathrm{Ball}_s$ with probability at least $1-\delta$, i.e., $\|\hat{\vw}_s - \vw^*\|^2_{\Sigma_s} \leq \beta_s$.
\end{lemma}
\begin{remark}
The design of $\mathrm{Ball}_s$ is similar to that in LinUCB \citep{abbasi2011improved} and QLinUCB \citep{WZL+23}, but our choice of $\beta_s$ is different. $\beta_s$ cannot be too small to lose the track of optimal parameter $\vw^*$ and it cannot be too large as it appears in final regret bound. Overall the confidence analysis ensures that $\beta_s=\tilde{O}(1)$ and slightly grows as the stage index $s$.
\end{remark}

\section{Experiments}\label{sec:exp}

\textbf{Experimental Setup.} We compare the performance of our Q-NLB-UCB algorithm with QLinUCB \citep{WZL+23}, Q-GP-UCB \citep{DLV+23} and QMCKernelUCB \citep{hikimaquantum}. According to \cite{hikimaquantum}, the only difference between Q-GP-UCB and QMCKernelUCB is a trade-off parameter $\eta$ balancing the precision of quantum amplitude estimation and the noise in observations, and when $\eta = 1$, they become the same. In our experiments, we set $\eta = 0.1$ for QMCKernelUCB. 
To run Q-NLB-UCB, we set our parametric function model $f_\vw$ to be a two linear layer neural network with the sigmoid activation function.
Our implementation is based upon the sklearn, BoTorch \citep{balandat2020botorch}, and Qiskit \citep{qiskit2024}. 
To stay consistent with theoretical analysis, cumulative regret is also used to evaluate algorithms. In order to reduce the impact of randomness in algorithms, we repeat each experiment 5 times and report the mean and adjusted standard errors of cumulative regrets, i.e., $\mathrm{mean} \pm \mathrm{std}/\sqrt{5}$.

\noindent\textbf{High-Dimensional Synthetic Functions.}
We test all four algorithms on two functions, Rastrigin function and Styblinski-Tang function, defined in $30$-dimensional space $[-5,5]^{30}$. Q-NLB-UCB performs well as evidenced in Figure~\ref{fig:cr}(a)(b), where it consistently achieves the lowest cumulative regret across multiple runs, outperforming all other algorithms. It is not a surprise as Q-GP-UCB and QMCKernelUCB suffer from the curse of dimensionality and Q-GP-UCB was only tested in $1$-d and $2$-d settings in \cite{DLV+23}. QLinUCB performs the worst since it is designed to work on linear functions only while both two test functions here are highly non-linear.

\begin{table}[!htbp]
\centering
\begin{tabular}{lcc}
\toprule
Algorithms & Rastrigin & Styblinski-Tang \\   \midrule
  Q-GP-UCB    & 4629.6453  & 4139.2478  \\ \noalign{\smallskip}
  QMCKernelUCB &  3744.1179  & 2565.0690  \\\noalign{\smallskip}
  Q-NLB-UCB (ours)     & 861.2402   & 919.7602  \\
\bottomrule
\end{tabular}
\caption{Runtime (in seconds) on two synthetic functions}
\label{tab:2}
\end{table}

In addition, we report the runtime of three quantum bandit algorithms in seconds shown in Table~\ref{tab:2}. Among the three algorithms, our Q-NLB-UCB algorithm achieves a significantly low runtime, which again shows the efficiency of Q-NLB-UCB and validates our theoretical time complexity analysis in appendix. We don't list the runtime of QLinUCB because it is a linear bandit algorithm running faster but not comparable to algorithms designed for non-linear optimization. 

\noindent\textbf{Real-World AutoML Tasks.}
We test all four quantum bandit algorithms on three different hyperparameter tuning tasks for Support Vector Machine (SVM), Multi-Layer Perceptron (MLP), and Gradient Boosting (GB). We are tuning $4$ hyperparameters in SVM, $8$ in MLP, and $11$ for GB. Each classifier is trained by different hyperparameter configurations and the goal is to maximize the validation accuracy on a hold-out set.
Due to page limit, we only show results of MLP in Figure~\ref{fig:cr}(c)(d) on both breast cancer and diabetes datasets, and readers are referred to appendix for similar results of SVM and GB. Again our Q-NLB-UCB algorithm outperforms all other algorithms by achieving significantly smaller regrets, demonstrating its strong potential for practical applications.

\section{Conclusion}\label{sec:con}

With the aid of quantum computing, recent works \citep{DLV+23,hikimaquantum} showed that new $O(\mathrm{poly} \log T)$ regret bound can be achieved in quantum non-linear bandit optimization, but their works heavily rely on the RKHS assumption which suffers from the curse of dimensionality. Real-world data usually sit in high-dimensional spaces, making their regret bounds vacuous.
In this paper, we develop the new Q-NLB-UCB algorithm which efficiently solves the problem in high-dimensional cases. The key design of Q-NLB-UCB involves quantum Monte Carlo mean estimation, parametric function approximation, and quantum fast-forward techniques, which all contribute to the new \emph{input dimension-free} regret bound of Q-NLB-UCB. Moreover, the choice of parametric functions can be generic, such as linear or quadratic functions, or even deep neural networks. Technically, our analysis of the new quantum non-linear regression oracle can be of independent interests in more quantum machine learning problems in the future.

\section*{Acknowledgments}

The authors would like to thank the anonymous reviewers for helpful comments that improved the final version of this paper.

\bibliography{references_camera}  

\newpage
\onecolumn
\appendix

\section{Auxiliary Lemmas}\label{sec:aux}
In this section, we show auxiliary lemmas and definitions that will be used later in proofs.

\begin{lemma}[Upper bound of summation of weights (Adapted from Lemma 2 in \cite{DLV+23})]\label{lem:sum_eps_squ}
Set $\epsilon_s = \|\nabla f_{\vx_s}(\hat{\vw}_0)\|_{\Sigma^{-1}_{s}}$, then it satisfies that
\begin{align*}
    \sum_{s=1}^{m} \frac{1}{\epsilon_s} \leq T \quad \mathrm{and} \quad \sum_{s=1}^{m} \frac{1}{\epsilon^2_s} \leq T^2.
\end{align*}
\end{lemma}
\begin{proof}
First we can lower bound the total number of rounds in $m$ stages. Following Eq. \eqref{eq:contradiction}, we have
\begin{align*}
    \sum_{s=1}^m \frac{C_1}{\epsilon_s} \log \frac{m}{\delta} \geq \sum_{s=1}^m \frac{1}{\epsilon_s} \geq \sqrt{\sum_{s=1}^m \frac{1}{\epsilon^2_s}}.
\end{align*}
Suppose that $\sum_{s=1}^{m} \frac{1}{\epsilon^2_s} \geq T^2$, then we have 
\begin{align*}
    \sum_{s=1}^m \frac{C_1}{\epsilon_s} \log \frac{m}{\delta} > T,
\end{align*}
which is a contradiction. Similarly we can prove both statements.
\end{proof}

\begin{lemma}[\citet{Wain19}]\label{lm:bernstein_to_sub-exp}
Let $X$ be a random variable with mean $\mu$ and variance $\sigma^2$. If $X$ satisfies the following Bernstein moment condition,
\[
\E[|X-\mu|^k] \leq \frac{\sigma^2}{2}k!\ \beta^{k-2}
\]
for some $\beta>0$ and all $k \geq 2$, then
\[
\E[e^{\lambda(X-\mu)}] \leq \exp\left( \frac{\lambda^2 \sigma^2}{2(1-\beta \lambda)} \right),\quad \forall \lambda \in \left[0, \frac{1}{\beta}\right).
\]
\end{lemma}

\section{Discussion on Quantum Oracle}\label{app:quantum_oracle}

\noindent \textbf{Quantum Access to Random Variables.} Unlike classical access to samples of random variables, in quantum realm, the corresponding distribution is being accessed by making query to quantum sampling oracle. 
\begin{definition}[Quantum sampling oracle]\label{def:sampling_oracle}
For a random variable $Y$ with (finite) sample space $\Omega$, its quantum sampling oracle $\mathcal{O}_Y$ is defined as
\begin{equation}\label{eq:QSO}
    \calO_Y : |0 \> \longmapsto \sum_{\vy \in \Omega } \sqrt{\Pr[Y= \vy]} |\vy\> \otimes |\psi_\vy\>,
\end{equation}
where $|\psi_\vy\>$ is an arbitrary quantum state for every $\vy$.
\end{definition}
The content in second quantum register can also be viewed as possible quantum garbage appeared during the implementation of the oracle. Observe that if we directly measure the output of $\mathcal{O}_Y$, it will collapse to a classical sampling access to $Y$ that returns a random sample $\vy$ with respect to probability $\Pr[Y= \vy]$. In particular, the quantum noisy function \Cref{eq:noisy_func_value} is an instance of the quantum sampling oracle. 

\paragraph{Feasibility and Practicality.}

This same type of quantum oracles were also used in previous work \cite{WZL+23,DLV+23,hikimaquantum,SZ23}. As discussed in \cite{DLV+23}, a quantum oracle is available when the learning environment is implemented by a quantum algorithm, which makes the setting of their quantum Bayesian optimization fairly general. Therefore, for example, such quantum oracles will arise naturally when quantum bandit algorithms are used to optimize hyperparameters of quantum neural networks and any other quantum machine learning models. Furthermore, there are standard techniques \cite{NC10} in theory for implementing the quantum analogs of classical algorithms. Specifically, if there is a classical circuit for the given classical oracle, its quantum version of the same asymptotic computational complexity can be built using standard technique. Thus, in some scenarios, quantum oracles can be considered as a quantum generalization of their classical counterparts, and then this framework can be applied to optimize the parameters of models implemented on a quantum computer or in quantum systems where the data itself is inherently quantum, as pointed out in \cite{DLV+23}. In particular, \cite{DLV+23} successfully implemented their algorithm in the IBM real quantum computer, which shows that the quantum oracle can be fully realized in the real world.

In practice, however, the current development of quantum hardware is still at a relatively early stage. The actual implementation feasibility depends on factors such as circuit size and the complexity of the required quantum gates. Therefore, we believe that our quantum oracle can be practically feasible in small-size problem settings where the oracle is given as an explicit circuit. 

\section{Full Proofs}\label{app:proof}

In this section, we show full proofs of all technical results in the main paper.

\subsection{Quantum Non-Linear Regression Oracle}\label{apdx_phaseI}

In this section, we first review the details and properties for step 1 of \Cref{alg:qra}, and the rest of the algorithm. At the end we provide the necessary theoretical analysis for the complete \Cref{alg:qra}.

\subsubsection{Proof Overview on $\mathsf{Oracle}$}\label{sec:oracle}
In this section, we describe the quantum regression oracle, denoted as $\mathsf{Oracle}$ in step 1 in \Cref{alg:qra}, that we need for solving a non-linear \emph{non-linear least square} problem. 

Let $D^{T_0} \vw$ denote a classical algorithm for solving a \emph{non-linear least square} problem where $\vw$ is an initial value of the solution and $D$ is a matrix representing one iteration of the algorithm. It is easy to see that Markov chain-based algorithms (e.g SGD) can be represented using this notation.

Before presenting the actual proof of \Cref{thm:phaseI}, we provide a high-level overview of the key technical aspects. The core idea in proving Lemma 5.1 is formulating the difference between expected risk and empirical risk as bounded variables satisfying a certain condition and then applying the Craig-Bernstein (CB) inequality \cite{Cra1933}. In the formulation, the variable is linked to the unknown parametric function within the optimization problem (i.e. $f_\vw(\vx)$), and the inequality is then utilized in a summation across a set of data samples. This approach is standard for classical algorithms. However, to fully harness the quantum advantage, we can employ a quantum sampling oracle on the dataset, as defined in \Cref{eq:QSO}. With such oracle, the random variables being used in the proof will be associated with the expected risk function $L(\vw)$, and then the summation in the application of CB inequality will be computed over the the number of iterations. Note that this approach still only yields a rate of $\tilde{O}(1/T_0)$. To obtain an improved rate of $\tilde{O}(1/T_0^2)$, we leverage the quadratic speed-up provided by quantum computing when implementing algorithms that satisfy the framework specified as in $D^{T_0} \vw$. While such quantum enhancements can be realized using quantum singular value transformation \cite{GSLW19}, we instead adopt the quantum fast-forward technique introduced in \cite{AS19}, following the description in \cite{AGJK20}.
\begin{theorem}[Formal rephrase of \Cref{thm:QFF_informal} in \cite{AGJK20}]\label{thm:QFF}
    Let $\epsilon \in (0,1), s \in [0,1]$ and $t \in \N$. Let $P$ be any transition matrix which defines a reversible Markov chain on state space $X$ and $D$ be its discriminant matrix. Also let $\mathsf{Q}$ be the cost of implementing a quantum walk step. There is a quantum algorithm with complexity $O(\mathsf{Q} \sqrt{t \log(1/\epsilon)}$ that takes input $|\overline{0}\>|\psi\> \in \mathrm{span}\{|\overline{0}\>|x\> : x \in X\}$, and outputs a state that is $\epsilon$-close to a state of the form
    \begin{equation}\label{eq:qff}
        |0\>^{\otimes a} |\overline{0}\> D^t|\psi\> + |\Gamma\>
    \end{equation}
    where $a=O(\log(t\log(1/\epsilon)))$, $|\overline{0}\> $ is some fixed reference state, and $|\Gamma\>$ is some garbage state that has no support on states containing $|0\>^{\otimes a} |\overline{0}\> $ in the first two registers.
\end{theorem}
Informally, this theorem shows that a $t$-step classical random walk/Markov chain can be closely approximated by quantum walk using only $\sqrt{t}$ steps. This implies that the same performance guarantees achieved by a $t$-step classical random walk can also be attained by a quantum algorithm with time complexity $\propto \sqrt{t}$. 

A more visual intuition on how the step 1 of \Cref{alg:q_go_ucb} works is pictured in \Cref{fig:step1}. It describes a high‑level picture of a quantum algorithm implementing $D^{T_0} \vw$ which denotes an execution of a classical algorithm solving a non‑linear least square problem. More concretely, the input qubits to the quantum algorithm can be divided into three different parts: parameter register, input data register, and auxiliary qubits. Then the complete quantum process before the measurement step gives $(UL)^{T_0}|w\rangle|0\rangle|0\rangle$ (here 0’s in the last two registers can represent multiple qubits), and the resulting state will have information of $D^{T_0} \vw$. The composite operator $UL$ makes one step of the quantum walk, like one update in SGD. During each step, operator $L$ transforms the input from the last step into a quantum state whose expected value yields the gradient, and $U$ performs the update process.

In quantum algorithms leveraging Markov chains for optimization problems, their inputs can be decomposed into three parts. The parameter register consists of qubits that store the parameter information. The input data register is dedicated to generating the distribution for data samples. The auxiliary qubits provide the workspace for temporary information produced during the quantum algorithms that are not part of the final output. Given an input state across those three registers, the quantum oracle $\calL$--whose detailed description will be discussed in the subsequent paragraph--transforms the input into a quantum state whose expectation value yields the gradient $\nabla \calL$. While techniques from \cite{SZ23} enable unbiased mean estimation for gradient extraction, we omit further details here as they lie outside our scope. The unitary $\matU$ then processes the output of $\calL$, implementing a single iteration of a Markov-chain-based update rule, say determined by $\matD$, for solving non-linear least square problems. Roughly speaking, $\matU$ combines two components: (1) an unbiased non-destructive multivariate mean estimation subroutine; and (2) a quantum walk step over the parameter space. The composite operation $\matU \calL$ constitutes one full iteration of the algorithm, as formalized in \Cref{thm:QFF}. Starting from the initial state, after $T_0$ iterations, i.e. $(\matU \calL)^{T_0} |\vw\>|\vzero\>|\vzero\>$, the final quantum state will encodes $\matD^{T_0} \vw$. To gain a quadratic speed-up, we apply the quantum fast-forward technique to $(\matU \calL)^{T_0}$. Consequently, the step 1 of \Cref{alg:q_go_ucb} approximates \Cref{eq:qff} with $\widetilde{O}(\sqrt{T_0})$ iterations.

\begin{figure}[t]
    \centering
\begin{quantikz}
\lstick{parameter register $\ket{\vw}$} & \gate[wires=3][1.2cm]{\calL} & \gate[wires=3][1.2cm]{\matU} & \ldots \qw & \gate[wires=3][1.2cm]{\calL} & \gate[wires=3][1.2cm]{\matU} & \meter{} \\
\lstick{input data register $\ket{\vzero}$} &  & & \ldots \qw  & & & \qw &\\
\lstick{auxiliary qubits $\ket{\vzero}$}& & & \ldots \qw & & & \qw &
\end{quantikz}

    \caption{High-level description of a quantum algorithm computing $D^{T_0} \vw$}
    \label{fig:step1}
\end{figure}

Since the quantum oracle $\calL$ plays a pivotal role in the process, we now give a detailed discussion of its implementation requirements and constraints. To formalize its purpose, without loss of generality, consider the gradient of loss function $\E_{\vx \sim D} [(f_\vw(\vx) - f_{\vw^*}(\vx))^2]$, where the parametric function $f$ can be viewed as a function with input $(\vw, \vx)$. By standard differentiation, its gradient is $\E_{\vx \sim D} [2(f_\vw(\vx) - f_{\vw^*}(\vx)) \nabla f_\vw(\vx)] $. The oracle $\calL$ must therefore prepare a quantum state whose expectation value encodes this gradient. Note that $\calL$ viewed as an operator maps $|\vw\>$ to a quantum state $\sum_{Y} \sqrt{\Pr[Y]}|\calA(Y)\>$ for some operations $\calA$, and its expectation gives $\E_Y[\calA(Y)]$. Hence, we need to design the necessary $\calA$ so that the multivariate mean value $\E_Y[\calA(Y)]$ produces $\E_{\vx \sim D} [2(f_\vw(\vx) - f_{\vw^*}(\vx)) \nabla f_\vw(\vx)]$. One caveat is that $f_{\vw^*}(\vx)$ is not accessible directly. Instead, we are given access to the quantum noisy function oracle \Cref{eq:noisy_func_value}. However, this limitation can be addressed by leveraging the fact that $\E[y] = f_{\vw^*}(\vx)$. Overall, we can construct the quantum oracle $\calL$ as follows.

Given the quantum noisy function oracle $\calO_\vx$ (as defined in \Cref{eq:noisy_func_value}) for all $\vx$ in $\calX$, we can obtain the following quantum oracle $\calL$:
\begin{align*}
\calL : |\vw\> |\vzero\> |\vzero\> & \longmapsto  \sum_{\vx} \sum_{\ve} \sqrt{\Pr[\vx]} \cdot \sqrt{\Pr[\ve]} |(f_\vw(\vx) - (f_{\vw^*}(\vx) + \ve)) \cdot \frac{\partial f_\vw(\vx)}{\partial w_1}\> \\
&\qquad \qquad \otimes \cdots |(f_\vw(\vx) - (f_{\vw^*}(\vx) + \ve)) \cdot \frac{\partial f_\vw(\vx)}{\partial w_d}\> \otimes |\mathrm{garbage}\>.
\end{align*}
The proof for this conversion is straightforward. By standard quantum computation principles, any classical circuit $g$ can be embedded into a quantum circuit $G$ such that: $|\vx\>|\vzero\> \stackrel{G}{\longmapsto} |\vx\>|g(\vx)\>$. Building on this, without loss of generality, we construct a quantum circuit that prepares a quantum state whose expectation value encodes the gradient of expected loss function $L= \E_\vx[(f_\vx(\vw) - f_\vx(\vw^*))^2]$. Three components enable this: (1) the parametric function $f$ is classically specified; (2) the data input distribution can be created from scratch efficiently; and (3) the value $f_\vx(\vw^*)$ can be obtained as the mean of the operator in \Cref{eq:noisy_func_value}. Following the standard embedding process mentioned above, these three components are combined into a unified quantum circuit acting on registers.

\subsubsection{Overview on \Cref{alg:qra}}\label{sec:rest}
Let $|\hat{\vw}_0\>$ be the quantum state returned by $\mathsf{Oracle}$ in step 1 and write $|\hat{\vw}_0\> = \sum_i^{d_\vw} a_i |i\>$. The rest of \Cref{alg:qra} is about retrieving $a_i$'s for $i = 1, \cdots, d_\vw$. One standard approach to achieve that is using quantum state tomography. However, in our case, we can use amplitude estimation instead because each $a_i \in [0,1]$. At the same we want to avoid repeatedly calling $\mathsf{Oracle}$ for obtaining more copies of $|\hat{\vw}_0\>$ since that will dramatically increase the number of queries to the quantum sampling oracle $\calO_\vx$, which will in turn cause extra cumulative regret.  Hence, we will a technique called \emph{non-destructive amplitude estimation} which can return both an estimate of $a_i$ and a copy of the input state $|\hat{\vw}_0\>$. There are multiple existing works in this field. \cite{RF23} listed multiple of them and we pick one among them. Below is a more formal statement of \Cref{thm:ndae_informal}.

\begin{theorem}[Non-destructive amplitude estimation in \cite{RF23}]\label{thm:ndae}
    Given one copy of a quantum state $|\psi\>$, a projector $P$, and $\epsilon, \delta \in [0,1/2]$. Let $a = \<\psi|P|\psi\>$. Then there is a quantum algorithm $\mathsf{NDAE}(|\psi\>, P, \epsilon, \delta)$ that outputs with probability at least $1-\delta$, an estimate $\tilde{a}$ and a copy of $|\psi\>$ such that $|\tilde{a} - a| \leq \epsilon$ and it uses the reflections and rotations on $|\psi\>$ and $P$ $\tilde{\calO}(\frac{1}{\epsilon \sqrt{\delta}})$ times.
\end{theorem}
\begin{remark}
    Given a projector $P$, its reflection can be defined as $2P - I$ or $I - 2P$, and its rotation is defined as $e^{2\theta (2P - I)}$ or $e^{2\theta (I-2P)}$ for arbitrary phases $\theta$. Note that the projector for a quantum state $|\psi\>$ can be represented as $|\psi\> \<\psi|$. Hence, the reflections and rotations for state $|\psi\>$ can be defined similarly. Since neither reflections or rotations require queries to the quantum sampling oracles $\calO_\vs$, the usage of them will not contribute to cumulative regret. Hence, their implementations are outside the scope of our work. Interested readers can refer to \cite{RF23} for details on how to implement them using the given state and projector.
\end{remark}

With \textsf{NDAE}, we can retrieve an estimate for each $a_i$ for $i = 1, \cdots, d_\vw$ with only one copy of $|\hat{\vw}_0\>$, thus avoiding the expenses caused by repeatedly solving a non-linear least square problem. Overall, the correctness of \Cref{alg:qra} is a result of \Cref{thm:phaseI} and \Cref{thm:ndae}, which can be summerized as follows.

\begin{theorem}[Restatement of \Cref{thm:qro}]\label{thm:qro_restate}
    Suppose Assumptions \ref{asm:real}, \ref{asm:bounded}, and \ref{asm:loss} hold. Then \Cref{alg:qra} returns with probability $1-\delta /2$, a classical vector of an estimate $\hat{\vw}_0$ that satisfies
    \[
    \|\hat{\vw}_0 - \vw^*\|_2 \leq \frac{\sqrt{C d_\vw \iota}}{T_0},
    \]
    where $\iota$ is the logarithmic term depending on $T_0, C_h, 1/\delta$ and $T_0$ satisfies
    $T_0^2 \geq C d_\vw \iota \max\left\{ \frac{\mu^{\gamma/(2-\gamma)}}{\tau^{2/(2-\gamma)}}, \frac{2C^2_g}{\mu c^2} \right\}$.
\end{theorem}
Its proof is presented in the next section.

\subsubsection{Theoretical Analysis}
Recall that a quantum sampling oracle enables the access to a superposition that entirely encodes a distribution of finite size. The following notations will be found convenient in constructing our proofs. Let $Y(Z) = f_{\vw^*}(Z)$ and define $F_\vw(X) := \sum_{Z \in D(X)} (f_\vw(Z) - Y(Z))^2 = \sum_{i=1}^M (f_\vw(Z_i) - Y(Z))^2 $ where $D(X)$ is some probability distribution with finite sample space $\Omega \subseteq \R^d$ centered at $X$ of size $|\Omega| = M$, $L(f_\vw) = \E_X F_\vw(X)$, and $\hat{L}(f_\vw)= 1/T_0\sum_{i=1}^{T_0} F_\vw(X_i)$. It is easy to check that $\E_X[\hat{L}(f_\vw)] = L(f_\vw)$. We also assume that the output value of $f$ is contained in a interval of length $b \geq 1$. Note that if $M = 1$, then the quantum sampling oracle becomes a classical oracle. Hence, without loss of generality, we can have $M \gg 1$.  

\begin{remark}
    In classical setting, $Y(Z) = f_{\vw^*}(Z) + \eta$, and a $Y(Z)$ is returned when querying with input $Z$, and one can estimate $f_{\vw^*}(Z)$ by Chernoff bound. However, in quantum realm, one can extract $f_{\vw^*}(Z)$ while still maintaining it as a quantum superposition instead of a classical sample by mean estimation algorithms \cite{Mon15,SZ23} without performing the measurement. Therefore, it is not a strong relaxation to have $Y(Z) = f_{\vw^*}(Z)$ in the analysis.
\end{remark}

\begin{lemma}\label{lm:squared_error_loss}
Given an oracle that accesses a dataset $\{Z,f_{\vw^*}(Z)\}_{Z\in D(X)}$ where $D$ is some distribution with finite sample space $\Omega \subseteq \R^d$ centered at $X$ when queried with input $X$. Let $\calF$ be a finite function class satisfying $\calF \subset \{ f:[0,1]^d \rightarrow [-b, b] \}$ for some $b \geq 1$ and define empirical risk minimizer (ERM) $\hat{f}_{T_0} := \argmin_{f \in \calF} \left\{ \hat{L}(f) \right\}.$ Then with probability at least $1-\delta$, the following holds
\[
L(\hat{f}_{T_0})\leq \frac{1+\alpha}{1-\alpha}L(\tilde{f}) +  2 \frac{\log |\calF| + \log 2/\delta}{(1-\alpha)T_0 \epsilon}
\]
where $\epsilon < \frac{c_1^{k/(k-2)}}{4Lb^2}$, $\alpha = \epsilon b^2$, and $\tilde{f} \in \calF$ is arbitrary.
\end{lemma}
\begin{proof}
    The proof follows a structure similar to that in \cite{Now09_12}, but some extra work is needed for our application. Define variables $U_i = - F_\vw(X_i)$. Then $L(f_\vw) -\hat{L}(f_\vw) = \frac{1}{T_0} \sum_{i=1}^{T_0} (U_i - \E_X[U_i])$. In order to apply the CB inequality \cite{Cra1933} as in \cite{Now09_12}, we need to (1) verify that the variables $U_i$ satisfy the Berstein's moment condition
    \[
    \E_{X_i}[|U_i - \E_{X_i}[U_i]|^k] \leq \frac{\var{U_i}}{2}k!\ h^{k-2}
    \]
    for some $h > 0$ and all $k \geq 2$; (2) find an upper bound for $\var{U_i}$. They can be achieved by the following two claims.
    \begin{claim}\label{clm:monent_condition}
        The Berstein's moment condition holds with $h = \frac{2\sqrt{M}b}{c_1}$ with some constant $c_1$ where $1 \leq c_1 < \sqrt{\var{U}}/\beta$, $\sqrt{\var{U}}\neq c_1 \beta$, and $\beta = \frac{2b^2}{3}$.
    \end{claim}
    \begin{claim}\label{clm:U_variance}
        \[
        \var{U_i} \leq b^2 L(F_\vw).
        \]
    \end{claim}
Suppose \Cref{clm:monent_condition} and \Cref{clm:U_variance} hold. By applying the CB inequality with certain values $\epsilon$ and $c$, which will be determined later, we obtain that, with probability at least $1-\delta$,
\[
L(f_\vw) - \hat{L}(f_\vw) \leq \frac{\log 1/\delta}{T_0 \epsilon} + \frac{\epsilon b^2 L(F_\vw)}{2(1-c)}
\]
for $0 < \epsilon h \leq c < 1$. Following the Kraft inequality and union bound trick on the finite set $\calF$ in \cite{Now09_12}, we have that, for any $\delta >0$,
\[
L(f_\vw) - \hat{L}(f_\vw) \leq \frac{\log |\calF| + \log 1/\delta}{T_0 \epsilon} + \frac{\epsilon b^2 L(f_\vw)}{2(1-c)},\quad \forall f_\vw \in \calF
\]
with probability at least $1-\delta$.
\begin{claim}\label{clm:alpha}
    Let $0< \epsilon < \frac{c_1}{4M b^2}$, $c = \epsilon h$, and $\alpha = \epsilon b^2$. Then $0< c<\frac{1}{2}$ and $\frac{\epsilon b^2}{2(1-c)}< \alpha < 1$.
\end{claim}

By \Cref{clm:alpha}, we rearrange to get
\[
(1-\alpha)L(f_\vw) \leq \hat{L}(f_\vw) + \frac{\log |\calF| + \log 1/\delta}{T_0 \epsilon} ,\quad \forall f_\vw \in \calF
\]
with probability at least $1-\delta$. Define 
\[
\hat{f}_{T_0} := \argmin_{f_\vw \in \calF} \left\{ \hat{L}(f_\vw) \right\}.
\]
Then with probability at least $1-\delta$, 
\begin{align*}
    (1-\alpha)L(\hat{f}_{T_0}) &\leq \hat{L}(\hat{f}_{T_0}) + \frac{\log |\calF| + \log 1/\delta}{T_0 \epsilon}, \\
    &\leq \hat{L}(\tilde{f}) + \frac{\log |\calF| + \log 1/\delta}{T_0 \epsilon},
\end{align*}
where $\tilde{f} \in \calF$ is arbitrary. Similarly, we apply CB inequality to $\hat{L}(\tilde{f} ) -L(\tilde{f} ) = \frac{1}{T_0} \sum_{i=1}^{T_0} -(U_i - \E_X[U_i])$ and attain
\[
\hat{L}(\tilde{f} ) -L(\tilde{f} ) \leq \alpha L(\tilde{f}) + \frac{\log |\calF| + \log 1/\delta}{T_0 \epsilon}
\]
with probability at least $1-\delta$. By union bound, we get
\[
L(\hat{f}_{T_0}) \leq \frac{1+\alpha}{1-\alpha}L(\tilde{f}) +  2 \frac{\log |\calF| + \log 1/\delta}{(1-\alpha)T_0 \epsilon}
\]
with probability at least $1-2\delta$, for any $\delta>0$.
\end{proof}

\begin{lemma}
    Suppose \Cref{asm:real} and \Cref{asm:bounded} hold. There is an absolute constant $C'$, such that after $T_0$ iterations in Step 1 of \Cref{alg:q_go_ucb}, with probability at least $1-\delta/2$, the quantum regression oracle returns an estimate $\hat{\vw}_0$ that satisfies 
    \[
    L(f_{\hat{\vw}_0}) \leq \frac{C' d_{\vw} \iota}{T_0},
    \]
    where $\iota $ is a logarithmic term depending on $T_0, C_h, 1/\delta$.
\end{lemma}
\begin{proof}
Its proof follows that of Lemma 5.1 in \cite{LW23} with minor changes. We provide it here for completeness.

Let $\tilde{\vw}, \widetilde{\calW}$ denote the ERM parameter and finite parameter class after applying covering number argument on $\calW$. By \Cref{lm:squared_error_loss}, we obtain that with probability at least $1-\delta/2$,
    \begin{align*}
        L(f_{\tilde{\vw}}) 
        &\leq \frac{1+\alpha}{1-\alpha}L(f_{\vw^*}) +  2 \frac{\log |\calF| + \log 4/\delta}{(1-\alpha)T_0 \epsilon}, \\
        &\leq 2 \frac{\log |\calF| + \log 4/\delta}{(1-\alpha)T_0 \epsilon},
    \end{align*}
where the first inequality follows from $\vw^* \in \calF = \widetilde{\calW} \cup \{\vw^*\}$. Our parameter class $\calW \subseteq [0,1]^{d_\vw}$, so by $\epsilon$-covering number argument, $\log(|\widetilde{\calW}|) = \log(1/\epsilon'^{d_\vw}) = d_\vw \log(1/\epsilon')$, and then we have with probability at least $1- \delta/2$,
    \[
    L(f_{\tilde{\vw}}) \leq C'' \frac{d_\vw \log(1/\epsilon') + \log 4/\delta}{T_0},
    \]
where $C''$ is a universal constant determined by $\alpha$ and $\epsilon$ picked in \Cref{lm:squared_error_loss}. By $(a+b)^2 \leq 2a^2 + 2b^2$,
    \begin{align*}
        L(f_{\hat{\vw}_0}) 
        &\leq 2\E_X[\sum_{Z\in D(X)} (f_{\hat{\vw}_0}(Z) - f_{\tilde{\vw}}(Z))^2] + 2L(f_{\tilde{\vw}}), \\
        &\leq 2 \epsilon'^2 C_h^2 + 2C'' \frac{d_\vw \log(1/\epsilon') + \log 4/\delta}{T_0},
    \end{align*}
    where the second inequality applies discretization error $\epsilon'$ and \Cref{asm:bounded}. By choosing $\epsilon' = \frac{1}{C_h \sqrt{T_0}}$, the bound above becomes
    \[
    \frac{2}{T_0} + \frac{C'' d_\vw \log (T_0 C_h^2)}{T_0} + \frac{2C'' \log 4/\delta}{T_0}
    \leq C' \frac{d_\vw \log(T_0 C_h^2) + \log(4/\delta)}{T_0},
    \]
    where we can take $C' = 2C''$ and assume $2 < C'' d_\vw \log (T_0 C_h^2)$. The proof completes by defining $\iota $ as the logarithmic term depending on $T_0, C_h, 1/\delta$.
\end{proof}

\begin{lemma}[Restatement of \Cref{lm:parameter_dist}]
    Suppose Assumptions \ref{asm:real}, \ref{asm:bounded}, and \ref{asm:loss} hold. There is an absolute value $C$ such that after $T_0$ iterations in step 1 of \Cref{alg:q_go_ucb} where $T_0$ satisfies
    \begin{align*}
        T_0 \geq C d_\vw \iota \max\left\{ \frac{\mu^{\gamma/(2-\gamma)}}{\tau^{2/(2-\gamma)}}, \frac{2C^2_g}{\mu c^2} \right\},
    \end{align*}
    with probability $1-\delta/2$, the quantum regression oracle returns an estimate $\hat{\vw}_0$ that satisfies 
    \begin{align*}
        \|\hat{\vw}_0 - \vw^*\|_2 \leq \sqrt{\frac{C d_\vw \iota}{T_0}},
    \end{align*}
    where $\iota$ is the logarithmic term depending on $T_0, C_h, 1/\delta$.
\end{lemma}
Its proof follows that of Theorem 5.2 in \cite{LW23} with no changes required. Note that the rate of $\|\hat{\vw}_0 - \vw^*\|_2$ is still $\tilde{O}(1/\sqrt{T_0})$ after $T_0$ iterations. To achieve the same rate of $\tilde{O}(1/\sqrt{T_0})$ after $\sqrt{T_0}$ iterations, we turn to \Cref{thm:QFF} which shows that the $T_0$ iterations can be approximately with only $\sqrt{T_0}$ iterations. 

\begin{theorem}[Restatement of \Cref{thm:phaseI}]\label{thm:phaseI_restate}
    Suppose Assumptions \ref{asm:real}, \ref{asm:bounded}, and \ref{asm:loss} hold. There is an absolute value $C$ such that after $\tilde{O}(T_0)$ iterations in step 1 of \Cref{alg:qra} where $T_0$ satisfies
    \begin{align*}
        T_0^2 \geq C d_\vw \iota \max\left\{ \frac{\mu^{\gamma/(2-\gamma)}}{\tau^{2/(2-\gamma)}}, \frac{2C^2_g}{\mu c^2} \right\},
    \end{align*}
    with probability $1-\delta/2$, the quantum regression oracle returns an estimate $\hat{\vw}_0$ that satisfies 
    \begin{align*}
        \|\hat{\vw}_0 - \vw^*\|_2 \leq \frac{\sqrt{C d_\vw \iota}}{T_0},
    \end{align*}
    where $\iota$ is the logarithmic term depending on $T_0, C_h, 1/\delta$.
\end{theorem}
\begin{proof}
    Suppose we select an algorithm used in this step following $D^{T_0} \vw$. Let $T_0$ be the number of iterations in this algorithm. Then it can be formulated $D^{T_0} \vw$ with a random initial parameter vector $\vw$, and the estimate output $\hat{\vw_0} = D^{T_0}\vw$. By \Cref{lm:parameter_dist}, we obtain that with probability at least $1-\delta/2$,
    \[
    \|\hat{\vw_0} - \vw^*\|_2^2 \leq \frac{C d_\vw \iota}{T_0}
    \]
    after $T_0$ iterations of exact calculation. 
    
    Next by \Cref{thm:QFF}, we have that there is a quantum algorithm $\calA$ can output a vector $\vw'$ that is $\sqrt{\epsilon_1}$-close to $D^{T_0}\vw$ with time complexity $O(\mathsf{Q}\sqrt{T_0 \log(1/\sqrt{\epsilon_1})}$ where $\mathsf{Q}$ is the cost of performing one quantum walk step. That implies that with $O(\sqrt{T_0 \log(1/\sqrt{\epsilon_1})}$ iterations, $\calA$ will return an estimate $\sqrt{\epsilon_1}$-close to $\hat{\vw}_0$. Overall, there is a quantum algorithm that takes input a random initial parameter $\vw$ and outputs an estimate $\hat{\vw_0}'$ such that with probability at least $1-\delta /2$,
    \[
    \|\hat{\vw_0}' - \vw^*\|_2^2 \leq \frac{C d_\vw \iota}{T_0} + \epsilon_1
    \]
    with $O(\sqrt{T_0 \log(1/\sqrt{\epsilon_1})}$ steps. Finally, the desired statement can be derived by setting $\epsilon_1 = 1/T_0$ and change of variable.
\end{proof}

\begin{proof}[Proof for \Cref{thm:qro_restate}]
Denote $|\hat{\vw}_0\> = \sum_i^{d_\vw} b_i |i\>$. It is easy to check that $\<\hat{\vw}_0 |P_i | \hat{\vw}_0\>$ gives $|b_i|^2$. Note that as a vector $\hat{\vw}_0 \in [0,1]^{d_\vw}$, which means $|b_i| =b_i$. 

Let $\delta' = \frac{\delta}{4 d_\vw}$. Then by \Cref{thm:ndae}, $\mathsf{NDAE}(|\hat{\vw}_0\>, P_i, \frac{1}{d_\vw T_0^2}, \delta')$ outputs with probability at least $1-\delta'$, an estimate $\tilde{a}_i$ of $b_i^2$ such that $|\tilde{a}_i - b_i^2| \leq \frac{1}{d_\vw T^2_0}$. Since both $\sqrt{\tilde{a}_i}$ and $b_i$ are nonnegative, we have 
\[
|\sqrt{\tilde{a}_i}- b_i | \leq \sqrt{|\tilde{a}_i - b_i^2|} \leq \frac{1}{\sqrt{d_\vw \cdot T_0^2 }} .
\]
Let $\tilde{\vw}_0 = (\sqrt{\tilde{a}_1}, \cdots, \sqrt{\tilde{a}_{d_\vw}})$. Then
\[
\|\tilde{\vw}_0 - \hat{\vw}_0 \|_2^2 = \sum_i^{d_\vw} |  \sqrt{\tilde{a}_i}- b_i |^2 \leq \frac{1}{T_0^2},
\]
and by triangle inequality, 
\[
\| \tilde{\vw}_0 - \vw^* \|_2 \leq \|\tilde{\vw}_0 - \hat{\vw}_0 \|_2 + \| \hat{\vw}_0 - \vw^*\|_2 \leq \frac{1+\sqrt{C d_\vw \iota}}{T_0}
\]
where the last equality follows from the combination of the upper bound on $\|\tilde{\vw}_0 - \hat{\vw}_0 \|_2$ and that on $\| \hat{\vw}_0 - \vw^*\|_2$ provided by \Cref{thm:phaseI_restate}.

Next, it is left to prove the success probability. First, step 1 in \Cref{alg:qra} calls $\mathsf{Oracle}$ which has $\calO(T_0)$ iterations and its success probability is at least $1-\delta/4$. Second, each $\mathsf{NDAE}(|\hat{\vw}_0\>, P_i, \frac{1}{d_\vw T_0^2}, \delta')$ succeeds with probability at $1- \delta'$. Hence, all the $d_\vw$ calls to $\mathsf{NDAE}$ outputs the correct estimates with probability at least 
\[
\left(1- \delta' \right)^{d_\vw} \geq 1 - d_\vw \cdot \delta' = 1 - \frac{\delta }{4}
\]
where the first equality follows from Bernoulli's inequality and the last step comes from the definition of $\delta'$.

Overall, \Cref{alg:qra} successfully outputs $\hat{\vw}_0 = (\tilde{a}_1, \cdots, \tilde{a}_{d_w})$ with probability at least $(1-\frac{\delta}{4})^2 \geq 1 -\frac{\delta}{2}$ as desired.

\end{proof}

\begin{proof}[Proof of \Cref{clm:monent_condition}]
    For each independent random variable $Z_i \sim D(X)$, define random variable $J_i = - (f(Z_i|\vw) - f(Z_i|\vw^*))^2$ and $V_i = J_i - \E[J_i]$. Then we can see that $U = \sum_{Z_i \sim D(X_i)} J_i$ with $\var{J_i}\leq b^2$, and $\E[V_i] = 0$. Since $D(X)$ has finite sample space $\Omega$ of size $|\Omega| = M$, we write $U = \sum_{\ell = 1}^M J_\ell$, and hence $\var{U} \leq Mb^2$.

    Next let $W = U - \E[U] = \sum_{\ell = 1}^M V_i$ and observe that $\E[W] = 0$. As in \cite{Now09_12}, note that each independent $J_i$ satisfies the following Bernstein condition
    \[
    \E[|J_i - \E[J_i]|^k] = \E[|V_i|^k] \leq \frac{\var{J_i}}{2}k!\ \beta^{k-2}
    \]
    for all $k\geq 2$ and with $\beta = \frac{2b^2}{3}$. Then by \Cref{lm:bernstein_to_sub-exp}, we have that for each $J_i$,
    \[
    \E[e^{\lambda V_i}] = \E[e^{\lambda(J_i - \E[J_i])}] \leq \exp\left( \frac{\lambda^2 \var{J_i}}{2(1-\lambda\beta)} \right),\quad \forall \lambda \in \left[0, \frac{1}{\beta}\right). 
    \]
    Next we put all $M$ independent $V_i$ together to get, with the same $\lambda$,
    \[
    \E[e^{\lambda W}] = \E[\prod_{i=1}^M e^{\lambda V_i}] = 
    \prod_{i=1}^M \E[e^{\lambda V_i}] 
    \leq \exp\left( \frac{\lambda^2 \var{U}}{2(1-\lambda\beta)} \right),
    \quad \forall \lambda \in \left[0, \frac{1}{\beta}\right).
    \]
    Note that this also holds when replacing $W$ with $-W$. Hence, for $m= 1,2,\cdots$
    \begin{equation*}
        \frac{\lambda^{2m}}{(2m)!} \E[W^{2m}] \leq \E[\frac{e^{\lambda W} + e^{-\lambda W}}{2}] \leq \exp\left( \frac{\lambda^2 \var{U}}{2(1-\lambda\beta)} \right),
    \end{equation*}
    where the first inequality follows from the Taylor series of exponential functions.

    Next we let $\lambda = \frac{c_1}{\sqrt{\var{U}}}$ with some constant $c_1 \geq 1$ such that $\sqrt{\var{U}} \neq c_1\beta$ and $\lambda < 1/\beta$. Then we can verify that $\exp\left( \frac{\lambda^2 \var{U}}{2(1-\lambda\beta)} \right)$ as a function of $\var{U}$ is decreasing for all $\var{U} \geq 0$ and $\sqrt{\var{U}} \neq c_1\beta$.
    
    With substitution and rearrangement, we can derive 
    \[
    \E[W^{2m}] \leq (2m)! \frac{1}{\lambda^{2m}} \leq (2m)!\ (\frac{\sqrt{\var{U}}}{c_1})^{2m}. 
    \]
    Next we apply Cauchy-Schwarz inequality to attain, for $m=1,2,\cdots$, 
    \begin{align*}
        \E[|W|^{2m+1}] &\leq \sqrt{\E[W^{2m}]} \sqrt{\E[W^{2m+2}]}, \nonumber \\
        &\leq  \sqrt{(2m)!\ (\frac{\sqrt{\var{U}}}{c_1})^{2m}} \sqrt{(2m+2)!\ (\frac{\sqrt{\var{U}}}{c_1})^{2m+2}}, \nonumber \\
        &\leq (2m+1)!\ (\frac{\sqrt{\var{U}}}{c_1})^{2m+1}. 
    \end{align*}
    Therefore, we have
    \begin{align*}
        \E[|U-\E[U]|^k] &= \E[|W|^k], \\
        &\leq k!\ (\frac{\sqrt{\var{U}}}{c_1})^{k}, \\
        &\leq k!\ \frac{1}{c_1^k}\var{U} (\sqrt{\var{U}})^{k-2}, \\
        &\leq \frac{\var{U}}{2}k!\ \frac{(2\sqrt{M}b)^{k-2}}{c_1^k}, \\
        &\leq \frac{\var{U}}{2}k!\ \left(\frac{2\sqrt{M}b}{c_1}\right)^{k-2},
    \end{align*}
    where the last step follows from $c_1 \geq 1$.
\end{proof}

\begin{proof}[Proof of \Cref{clm:U_variance}]
By the definition of $U_i$, we have
    \begin{align*}
        U_i^2 &= F_\vw(X_i)^2, \\
        &= \sum_{Z,Z' \in D(X)} (f_\vw(\Z) - f_{\vw^*}(\Z))^2 (f_\vw(\Z') - f_{\vw^*}(\Z'))^2, \\
        &\leq b^2 \sum_{Z \in D(X)} (f_\vw(\Z) - f_{\vw^*}(\Z))^2,
    \end{align*}
    and then 
    \[
    \var{U_i} \leq \E[U_i^2] = L(F_\vw).
    \]
\end{proof}

\begin{proof}[Proof of \Cref{clm:alpha}]
    When $\epsilon < \frac{c_1}{4M b^2}$, 
    \[
    c = \epsilon h = \epsilon \frac{2\sqrt{M}b}{c_1}
    < \frac{1}{2\sqrt{M}b} < \frac{1}{2}
    \]
    where the second step is by the value of $h$ from \Cref{clm:monent_condition} and the last step follows from the fact that $M,b \geq 1$. 

    Next, 
    \[
    \alpha = \epsilon b^2 < \frac{c_1}{4M}.
    \]
    It is straightforward that $\alpha > \frac{\epsilon b^2}{2(1-c)}$ since $c > 0$. Note that by \Cref{clm:monent_condition}, 
    \[
    1 \leq c_1 < \frac{\sqrt{\var{U}}}{\beta} \leq \frac{3\sqrt{M}}{2b},
    \]
    since $\sqrt{\var{U}} \leq \sqrt{M}b$ and $\beta = 2b^2/3$. Then
    \[
    \alpha < \frac{c_1}{4M} < \frac{3\sqrt{M}}{2b} \cdot \frac{1}{4M} = \frac{3}{8b\sqrt{M}} < 1 ,
    \]
    where the last step stems from the fact that $M,b \geq 1$.  
\end{proof}

\subsection{Regret Analysis}\label{app:regret}

\begin{theorem}[Restatement of Theorem \ref{thm:cr}]
Suppose Assumptions \ref{asm:real}, \ref{asm:bounded}, and \ref{asm:loss} hold. There is an absolute value $C$ such that after $\tilde{O}(T_0)$ iterations in Step 1 of \Cref{alg:q_go_ucb} where $T_0$ satisfies
\begin{align*}
    T^2_0 \geq C d_\vw \iota \max\left\{ \frac{\mu^{\gamma/(2-\gamma)}}{\tau^{2/(2-\gamma)}}, \frac{2C^2_g}{\mu c^2} \right\}
\end{align*}
with $\iota$ denoting a logarithmic term depending on $T_0, C_h, 1/\delta$. Then Algorithm \ref{alg:q_go_ucb} with parameters $T_0=\sqrt{T}, \lambda=T$ satisfies that with probability at least $1-\delta$,
\begin{align*}
R_T = O\left(d_w^2 \log^\frac{3}{2}(T) \log (d_w \log (T))\right).
\end{align*}
\end{theorem}

\begin{proof}
Since Algorithm \ref{alg:q_go_ucb} runs in multiple stages where the same action is played for multiple rounds, so we first focus on the instantaneous regret of one round in each stage and then calculate the cumulative regret in all stages. At stage $s$, the instantaneous regret of one round $r_s$ is defined as
\begin{align*}
r_s &= f_0(\vx^*) - f_0(\vx_s) = f_{\vx^*}(\vw^*) - f_{\vx_s}(\vw^*),
\end{align*}
where the second equation is due to Assumption \ref{asm:real}.
Recall that the selection process of $\vx_s$ in Algorithm \ref{alg:q_go_ucb} is
\begin{align*}
    \vx_s=\argmax_{\vx \in \cX} \max_{\vw \in \mathrm{Ball}_s} f_\vx(\vw),
\end{align*}
and we define $\tilde{\vw}$ to be the parameter that maximizes the function value at $\vx_s$, i.e, $\tilde{\vw} = \argmax_{\vw \in \mathrm{Ball}_s} f_{\vx_s}(\vw)$, then we have
\begin{align*}
r_s \leq f_{\vx_s}(\tilde{\vw}) - f_{\vx_s}(\vw^*)= (\tilde{\vw} - \vw^*)^\top \nabla f_{\vx_s}(\dot{\vw}),
\end{align*}
where the equation is by first order Taylor's theorem and $\dot{\vw}$ is a parameter lying between $\tilde{\vw}$ and $\vw^*$. Due to the convex structure of $\mathrm{Ball}_s$ for each stage $s$, it implies that $\dot{\vw} \in \mathrm{Ball}_s$. By adding and removing $\hat{\vw}_s$ and $\nabla f_{\vx_s}(\hat{\vw}_0)$, we have
\begin{align*}
r_s &\leq (\tilde{\vw} - \hat{\vw}_s + \hat{\vw}_s - \vw^*)^\top (\nabla f_{\vx_s}(\hat{\vw}_0) - \nabla f_{\vx_s}(\hat{\vw}_0) + \nabla f_{\vx_s}(\dot{\vw})),\\
&= (\tilde{\vw}-\hat{\vw}_s)^\top \nabla f_{\vx_s}(\hat{\vw}_0) + (\hat{\vw}_s - \vw^*)^\top \nabla f_{\vx_s}(\hat{\vw}_0)\\
&\qquad + (\tilde{\vw} - \hat{\vw}_s + \hat{\vw}_s - \vw^*)^\top (\nabla f_{\vx_s}(\dot{\vw}_s) - \nabla f_{\vx_s}(\hat{\vw}_0)),\\
&\leq \|\tilde{\vw}-\hat{\vw}_s\|_{\Sigma_s} \|\nabla f_{\vx_s}(\hat{\vw}_0)\|_{\Sigma_s^{-1}} + \|\hat{\vw}_s - \vw^*\|_{\Sigma_s} \|\nabla f_{\vx_s}(\hat{\vw}_0)\|_{\Sigma_s^{-1}} \\
&\qquad + (\tilde{\vw} - \hat{\vw}_s+ \hat{\vw}_s - \vw^*)^\top (\nabla f_{\vx_s}(\dot{\vw}_s) - \nabla f_{\vx_s}(\hat{\vw}_0)),
\end{align*}
where the last inequality is due to Holder's inequality.

Since both $\tilde{\vw}$ and $\vw^*$ are in $\mathrm{Ball}_s$ and $\epsilon_s = \|\nabla f_{\vx_s}(\hat{\vw}_0)\|_{\Sigma^{-1}_s}$, we have
\begin{align*}
r_s &\leq 2\sqrt{\beta_s} \epsilon_s + (\tilde{\vw} - \hat{\vw}_s+ \hat{\vw}_s - \vw^*)^\top (\nabla f_{\vx_s}(\dot{\vw}) - \nabla f_{\vx_s}(\hat{\vw}_0)).
\end{align*}
Again by first order Taylor's theorem where $\ddot{\vw}$ lies between $\dot{\vw}$ and $\hat{\vw}_0$, we have
\begin{align*}
r_s &\leq 2\sqrt{\beta_s} \epsilon_s + (\tilde{\vw}-\hat{\vw}_s+ \hat{\vw}_s - \vw^*)^\top \nabla^2 f_{\vx_s}(\ddot{\vw}) (\dot{\vw}-\hat{\vw}_0),\\
&= 2\sqrt{\beta_s} \epsilon_s + (\tilde{\vw}-\hat{\vw}_s+\hat{\vw}_s - \vw^*)^\top \nabla^2 f_{\vx_s}(\ddot{\vw}) (\dot{\vw}-\hat{\vw}_s+\hat{\vw}_s - \vw^* + \vw^* -\hat{\vw}_0),\\
&= 2\sqrt{\beta_s} \epsilon_s + (\tilde{\vw}-\hat{\vw}_s+\hat{\vw}_s - \vw^*)^\top \Sigma^\frac{1}{2}_s \Sigma^{-\frac{1}{2}}_s \nabla^2 f_{\vx_s}(\ddot{\vw}) \Sigma^\frac{1}{2}_s \Sigma^{-\frac{1}{2}}_s  (\dot{\vw}-\hat{\vw}_s+\hat{\vw}_s - \vw^* + \vw^* -\hat{\vw}_0),\\
&= 2\sqrt{\beta_s} \epsilon_s + (\tilde{\vw}-\hat{\vw}_s)^\top \Sigma^\frac{1}{2}_s \Sigma^{-\frac{1}{2}}_s \nabla^2 f_{\vx_s}(\ddot{\vw}) \Sigma^{-\frac{1}{2}}_s \Sigma^\frac{1}{2}_s (\dot{\vw}-\hat{\vw}_s) \\
&\qquad + (\tilde{\vw}-\hat{\vw}_s)^\top \Sigma^\frac{1}{2}_s \Sigma^{-\frac{1}{2}}_s \nabla^2 f_{\vx_s}(\ddot{\vw}) \Sigma^{-\frac{1}{2}}_s \Sigma^\frac{1}{2}_s (\hat{\vw}_s - \vw^*)\\
&\qquad + (\tilde{\vw}-\hat{\vw}_s)^\top \Sigma^\frac{1}{2}_s \Sigma^{-\frac{1}{2}}_s \nabla^2 f_{\vx_s}(\ddot{\vw}) \Sigma^{-\frac{1}{2}}_s \Sigma^\frac{1}{2}_s (\vw^* - \hat{\vw}_0) \\
&\qquad + (\hat{\vw}_s - \vw^*)^\top \Sigma^\frac{1}{2}_s \Sigma^{-\frac{1}{2}}_s \nabla^2 f_{\vx_s}(\ddot{\vw}) \Sigma^{-\frac{1}{2}}_s \Sigma^\frac{1}{2}_s (\dot{\vw}-\hat{\vw}_s) \\
&\qquad + (\hat{\vw}_s-\vw^*)^\top \Sigma^\frac{1}{2}_s \Sigma^{-\frac{1}{2}}_s \nabla^2 f_{\vx_s}(\ddot{\vw}) \Sigma^{-\frac{1}{2}}_s \Sigma^\frac{1}{2}_s (\hat{\vw}_s- \vw^*) \\
&\qquad + (\hat{\vw}_s-\vw^*)^\top \Sigma^\frac{1}{2}_s \Sigma^{-\frac{1}{2}}_s \nabla^2 f_{\vx_s}(\ddot{\vw}) \Sigma^{-\frac{1}{2}}_s \Sigma^\frac{1}{2}_s (\vw^*-\hat{\vw}_0),\\
&\leq 2\sqrt{\beta_s} \epsilon_s + \|(\tilde{\vw}-\hat{\vw}_s)^\top \Sigma^\frac{1}{2}_s\|_2 \|\Sigma^{-\frac{1}{2}}_s \nabla^2 f_{\vx_s}(\ddot{\vw}) \Sigma^{-\frac{1}{2}}_s\|_\mathrm{op} \|\Sigma^\frac{1}{2}_s (\dot{\vw}-\hat{\vw}_s)\|_2\\
&\qquad + \|(\tilde{\vw}-\hat{\vw}_s)^\top \Sigma^\frac{1}{2}_s\|_2 \|\Sigma^{-\frac{1}{2}}_s \nabla^2 f_{\vx_s}(\ddot{\vw}) \Sigma^{-\frac{1}{2}}_s \|_\mathrm{op} \|\Sigma^\frac{1}{2}_s (\hat{\vw}_s - \vw^*)\|_2\\
&\qquad + \|(\tilde{\vw}-\hat{\vw}_s)^\top \Sigma^\frac{1}{2}_s\|_2 \|\Sigma^{-\frac{1}{2}}_s \nabla^2 f_{\vx_s}(\ddot{\vw}) \Sigma^{-\frac{1}{2}}_s \|_\mathrm{op} \|\Sigma^\frac{1}{2}_s (\vw^* - \hat{\vw}_0)\|_2\\
&\qquad + \|(\hat{\vw}_s - \vw^*)^\top \Sigma^\frac{1}{2}_s\|_2 \|\Sigma^{-\frac{1}{2}}_s \nabla^2 f_{\vx_s}(\ddot{\vw}) \Sigma^{-\frac{1}{2}}_s \|_\mathrm{op} \|\Sigma^\frac{1}{2}_s (\dot{\vw}-\hat{\vw}_s)\|_2\\
&\qquad + \|(\hat{\vw}_s-\vw^*)^\top \Sigma^\frac{1}{2}_s\|_2 \| \Sigma^{-\frac{1}{2}}_s \nabla^2 f_{\vx_s}(\ddot{\vw}) \Sigma^{-\frac{1}{2}}_s \|_\mathrm{op} \|\Sigma^\frac{1}{2}_s (\hat{\vw}_s- \vw^*)\|_2\\
&\qquad + \|(\hat{\vw}_s-\vw^*)^\top \Sigma^\frac{1}{2}_s\|_2 \| \Sigma^{-\frac{1}{2}}_s \nabla^2 f_{\vx_s}(\ddot{\vw}) \Sigma^{-\frac{1}{2}}_s \|_\mathrm{op} \|\Sigma^\frac{1}{2}_s (\vw^*-\hat{\vw}_0)\|_2,
\end{align*}
where the second line is by adding and removing $\hat{\vw}_s$ and $\vw^*$ and the last inequality is due to Holder's inequality. Since the center of $\mathrm{Ball}_s$ is $\hat{\vw}_s$ and $\vw^*, \tilde{\vw}, \dot{\vw} \in \mathrm{Ball}_s$, then we have
\begin{align*}
r_s &\leq 2\sqrt{\beta_s} \epsilon_s + 4 \beta_s \|\Sigma^{-\frac{1}{2}}_s \nabla^2 f_{\vx_s}(\ddot{\vw}) \Sigma^{-\frac{1}{2}}_s\|_\mathrm{op} + 2 \sqrt{\beta_s} \|\Sigma^{-\frac{1}{2}}_s \nabla^2 f_{\vx_t}(\ddot{\vw}) \Sigma^{-\frac{1}{2}}_s\|_\mathrm{op} \|\Sigma^\frac{1}{2}_s (\vw^* -\hat{\vw}_0)\|_2, \\
&\leq 2\sqrt{\beta_s} \epsilon_s + 4 \beta_s \|\Sigma^{-\frac{1}{2}}_s \nabla^2 f_{\vx_s}(\ddot{\vw}) \Sigma^{-\frac{1}{2}}_s\|_\mathrm{op} + 2 \sqrt{\beta_s} \|\Sigma^{-\frac{1}{2}}_s \nabla^2 f_{\vx_t}(\ddot{\vw}) \Sigma^{-\frac{1}{2}}_s\|_\mathrm{op} \|\Sigma^\frac{1}{2}_s\|_\mathrm{op} \|(\vw^* -\hat{\vw}_0)\|_2, \\
&\leq 2\sqrt{\beta_s}\epsilon_s + \frac{4\beta_s C_h}{\lambda} + \frac{2\sqrt{\beta_s}C_h C_0}{T_0 \sqrt{\lambda}},
\end{align*}
where the second inequality is again due to Holder's inequality and the last inequality is by Assumption \ref{asm:bounded}, the choice of $\Sigma_s$, and convergence guarantee of $\hat{\vw}_0$.

Recall that in stage $s$, the algorithm plays actions $x_s$ for $\frac{C_1}{\epsilon_s}\log\frac{m}{\delta}$ rounds, therefore, the cumulative regret in stage $s$ is bounded as
\begin{align*}
\frac{C_1}{\epsilon_s} \left( 2\sqrt{\beta_s}\epsilon_s + \frac{4\beta_s C_h}{\lambda} + \frac{2\sqrt{\beta_s}C_h C_0}{T_0 \sqrt{\lambda}} \right) \log \left(\frac{m}{\delta}\right).
\end{align*}

In total, there are $m$ stages, so the cumulative regret is bounded as
\begin{align*}
R_T &\leq \sum_{s=1}^m \frac{C_1}{\epsilon_s} \left( 2\sqrt{\beta_s}\epsilon_s + \frac{4\beta_s C_h}{\lambda} + \frac{2\sqrt{\beta_s}C_h C_0}{T_0 \sqrt{\lambda}} \right) \log \left(\frac{m}{\delta}\right),\\
&= 2 C_1 \log \left(\frac{m}{\delta}\right) \sum_{s=1}^m \left(\sqrt{\beta_s} + \frac{2\beta_s C_h}{\epsilon_s \lambda} + \frac{\sqrt{\beta_s}C_h C_0}{\epsilon_s T_0 \sqrt{\lambda}} \right).
\end{align*}
Recall that our choice of $\beta_s$ is
\begin{align*}
   \beta_s = 3d_w s + \frac{3\lambda C_0^2}{T_0^2} + \frac{3C^2_h C_0^2 s T^2}{4 T_0^4},
\end{align*}
which is increasing in $s$, so $\beta_s \leq \beta_m$ where $m$ is the number of stages. Therefore, we have
\begin{align*}
R_T &\leq 2 C_1 \log \left(\frac{m}{\delta}\right) \left(m\sqrt{\beta_m} + \left( \frac{2\beta_m C_h}{\lambda} + \frac{\sqrt{\beta_m} C_h C_0}{T_0 \sqrt{\lambda}} \right) \sum_{s=1}^m \frac{1}{\epsilon_s}\right),\\
&\leq 2 C_1 \log \left(\frac{m}{\delta}\right) \left(m\sqrt{\beta_m} + \frac{2\beta_m C_h T}{\lambda} + \frac{\sqrt{\beta_m} C_h C_0 T}{T_0 \sqrt{\lambda}} \right),
\end{align*}
where the second inequality is due to Lemma \ref{lem:sum_eps_squ}.

Now plug in our choices $T_0=\sqrt{T}, \lambda=T$ (they are reverse-engineered to make the analysis work) and maximum stage number $m=d_w \log (\frac{C^2_g T}{d_w}+1)$ (Lemma \ref{lem:rounds}), and we have
\begin{align*}
R_T &\leq 2 C_1 \log \left(\frac{m}{\delta}\right) \left(m\sqrt{\beta_m} + 2\beta_m C_h + \sqrt{\beta_m} C_h C_0 \right),\\
\beta_m &= 3d_w m + 3C^2_0 + \frac{3C^2_h C^2_0 m}{4} = O(d_w m).
\end{align*}
Reorganize them and we have
\begin{align*}
R_T &= O \left( (m\sqrt{d_w m} + d_w m) \log \left(\frac{m}{\delta}\right) \right),\\
&= O \left((d_w\log(T))^\frac{3}{2} \sqrt{d_w} + d^2_w \log(T))\log(d_w \log(T))\right),\\
&=O\left(d_w^2 \log^\frac{3}{2}(T) \log (d_w \log (T))\right),
\end{align*}
which completes the proof.
\end{proof}

\begin{lemma}[Restatement of Lemma \ref{lem:rounds}]
The Algorithm \ref{alg:q_go_ucb} runs at most $m = d_w \log\left( \frac{C_g^2 T^2}{d_w \lambda} + 1 \right)$ stages.
\end{lemma}
\begin{proof}
The proof follows the outline of that for Lemma 2 in \cite{WZL+23}. First, we show that $\det(\Sigma_s) = 2 \det(\Sigma_{s-1}), \forall s = 1,2, \cdots, m$.
\begin{align*}
        \det(\Sigma_s) 
        &= \det \left( \Sigma_{s-1} + \frac{1}{\epsilon_s^2}  \nabla f_{\vx_s}(\hat{\vw}_0) \nabla f_{\vx_s}(\hat{\vw}_0)^\top \right), \\
        &= \det \left( \Sigma_{s-1}^{1/2}  \left( \matI_{d_w} + \frac{1}{\epsilon_s^2} \Sigma_{s-1}^{-\frac{1}{2}} \nabla f_{\vx_s}(\hat{\vw}_0) \nabla f_{\vx_s}(\hat{\vw}_0)^\top \Sigma_{s-1}^{-\frac{1}{2}} \right)   \Sigma_{s-1}^\frac{1}{2} \right), \\
        &= \det(\Sigma_{s-1}) \det\left( \matI + \frac{1}{\epsilon_s^2} \Sigma_{s-1}^{-\frac{1}{2}} \nabla f_{\vx_s}(\hat{\vw}_0) \nabla f_{\vx_s}(\hat{\vw}_0)^\top \Sigma_{s-1}^{-\frac{1}{2}}   \right), \\
        &= \det(\Sigma_{s-1}) \left( 1 + \frac{1}{\epsilon_s^2} \nabla f_{\vx_s}(\hat{\vw}_0)^\top \Sigma_{s-1}^{-\frac{1}{2}} \Sigma_{s-1}^{-\frac{1}{2}} \nabla f_{\vx_s}(\hat{\vw}_0) \right), \\
        &= \det(\Sigma_{s-1}) \left( 1+ \frac{1}{\epsilon_s^2} \| \nabla f_{\vx_s}(\hat{\vw}_0) \|^2_{\Sigma_{s-1}^{-1}} \right), \\
        &= 2 \det(\Sigma_{s-1}),
\end{align*}
where the fourth line follows from the matrix determinant lemma: $\det(\matA + \vu \vv^\top) = (1 + \vv^\top \matA^{-1} \vu) \det(\matA)$ where $\matA$ is an invertible square matrix and $\vu, \vv$ are column vectors. Thus, $\det(\Sigma_m) = 2^m \det(\Sigma_0) = 2^m \lambda^{d_w}$.

On the other hand, note that
\begin{align*}
    \Sigma_m = \lambda \matI_{d_w} + \sum^m_{s=1}  \frac{1}{\epsilon_s^2} \nabla f_{\vx_s}(\hat{\vw}_0) \nabla f_{\vx_s}(\hat{\vw}_0)^\top,
\end{align*}
and hence by cyclic property of matrix trace,
    \[
    \tr(\Sigma_m) = d_w \lambda +  \sum^m_{s=1}  \frac{\|\nabla f_{\vx_s}(\hat{\vw}_0) \|^2_2}{\epsilon_s^2}. 
    \]
Then by Assumption \ref{asm:bounded} and the trace-determinant inequality: $d_w \cdot \det(\matA)^\frac{1}{d_w} \leq \tr(\matA)$, we have
    \[
    d_w \cdot \lambda 2^\frac{m}{d_w} \leq d_w \lambda + \sum^m_{s=1}  \frac{\|\nabla f_{\vx_s}(\hat{\vw}_0) \|^2_2}{\epsilon_s^2} 
    \leq d_w \lambda + \sum^m_{s=1}  \frac{C_g^2}{\epsilon_s^2} ,
    \]
and hence
    \[
    \sum^m_{s=1}  \frac{1}{\epsilon_s^2} \geq \frac{d_w \lambda}{C_g^2}(2^\frac{m}{d_w} - 1).
    \]
Note that during each stage $s$, we query the quantum oracle for $\frac{C_1}{\epsilon_s} \log \frac{m}{\delta}$ times. Then with $C_1 > 1$ and $\delta \in (0, 1/2]$, we have
\begin{align}
    \sum^m_{s=1} \frac{C_1}{\epsilon_s} \log \frac{m}{\delta} 
    \geq \sum^m_{s=1}  \frac{1}{\epsilon_s}
    \geq \sqrt{\sum^m_{s=1}  \frac{1}{\epsilon_s^2}}
    \geq \frac{1}{C_g}\sqrt{d_w \lambda(2^\frac{m}{d_w} - 1)}.\label{eq:contradiction}
\end{align}
Now we derive an upper bound on $m$ by contradiction. Suppose $m > d_w \log\left( \frac{C_g^2 T^2}{d_w \lambda} + 1\right)$. This implies that
\begin{align*}
     \sum^m_{s=1} \frac{C_1}{\epsilon_s} \log \frac{m}{\delta} \geq \frac{1}{C_g}\sqrt{d_w \lambda(2^\frac{m}{d_w} - 1)} > T,
\end{align*}
which is a contradiction. Therefore, Q-NLB-UCB algorithm has at most $ d_w \log\left( \frac{C_g^2 T^2}{d_w\lambda} + 1\right)$ stages.
\end{proof}

\subsection{Confidence Analysis}\label{app:confidence}
\noindent \textbf{Proof Sketch.} The proof has two steps. In the first step, we solve the regression problem defined in Eq. \eqref{eq:opt_inner} to get a closed form solution of $\hat{\vw}_s$. Then in the second step, we upper bound multiple terms in $\|\hat{\vw}_s - \vw^*\|^2_{\Sigma_s}$ and the final upper bound is chosen as the valid $\beta_s$.

\begin{lemma}[Restatement of Lemma \ref{lem:beta}]
Suppose Assumptions \ref{asm:real}, \ref{asm:bounded}, and \ref{asm:loss} hold and $\beta_s$ is chosen as Eq. \eqref{eq:beta}. Then with parameters $T_0=\sqrt{T},\lambda=T$ in each stage $s$ in Algorithm \ref{alg:q_go_ucb}, the optimal parameter $\vw^*$ is trapped in confidence ball $\mathrm{Ball}_s$ with probability at least $1-\delta$, i.e.,
    \begin{align*}
        \|\hat{\vw}_s - \vw^*\|^2_{\Sigma_s} \leq \beta_s.
    \end{align*}
\end{lemma}
\begin{proof}
By setting the derivative of objective function in Eq. \eqref{eq:opt_inner} w.r.t. $\vw$ as $0$, we find the optimal criterion is
\begin{align*}
0= \lambda (\hat{\vw}_s - \hat{\vw}_0) + \sum_{i=0}^{s-1} \frac{1}{\epsilon^2_i}((\hat{\vw}_s - \hat{\vw}_0)^\top \nabla f_{\vx_i}(\hat{\vw}_0) + f_{\vx_i}(\hat{\vw}_0) - y_i ) \nabla f_{\vx_i}(\hat{\vw}_0).
\end{align*}
Rearrange the equation and we have
\begin{align*}
\lambda (\hat{\vw}_s- \hat{\vw}_0) + \sum_{i=0}^{s-1} \frac{1}{\epsilon_i^2}(\hat{\vw}_s - \hat{\vw}_0)^\top \nabla f_{\vx_i}(\hat{\vw}_0) \nabla f_{\vx_i}(\hat{\vw}_0) = \sum_{i=0}^{s-1} \frac{1}{\epsilon_i^2} (y_i - f_{\vx_i}(\hat{\vw}_0) ) \nabla f_{\vx_i}(\hat{\vw}_0),\\
\lambda (\hat{\vw}_s - \hat{\vw}_0) + \sum_{i=0}^{s-1}  \frac{1}{\epsilon_i^2} \hat{\vw}^\top_s \nabla f_{\vx_i}(\hat{\vw}_0) \nabla f_{\vx_i}(\hat{\vw}_0) = \sum_{i=0}^{s-1} \frac{1}{\epsilon_i^2} \left( \hat{\vw}^\top_0 \nabla f_{\vx_i}(\hat{\vw}_0) + y_i - f_{\vx_i}(\hat{\vw}_0) \right) \nabla f_{\vx_i}(\hat{\vw}_0),\\
\left(\lambda \matI + \sum_{i=0}^{s-1}\frac{1}{\epsilon^2_i} \nabla f_{\vx_i}(\hat{\vw}_0) \nabla f_{\vx_i}(\hat{\vw}_0)^\top \right) \hat{\vw}_s  -\lambda \hat{\vw}_0 = \sum_{i=0}^{s-1} \frac{1}{\epsilon_i^2} \left(\hat{\vw}^\top_0 \nabla f_{\vx_i}(\hat{\vw}_0) + y_i - f_{\vx_i}(\hat{\vw}_0)\right)\nabla f_{\vx_i}(\hat{\vw}_0),
\end{align*}
where the third line is due to definition of observation noise $\epsilon_i$. By our choice of $\Sigma_s$ (Eq. \eqref{eq:sigma_s}) which is inevitable, we have
\begin{align*}
\Sigma_s \hat{\vw}_s  &= \lambda \hat{\vw}_0 + \sum_{i=0}^{s-1} \frac{1}{\epsilon_i^2} \left( \hat{\vw}^\top_0 \nabla f_{\vx_i}(\hat{\vw}_0) + y_i - f_{\vx_i}(\hat{\vw}_0)\right)\nabla f_{\vx_i}(\hat{\vw}_0),
\end{align*}
And the closed form solution of $\hat{\vw}_s$ is shown as:
\begin{align*}
\hat{\vw}_s = \Sigma^{-1}_s \left( \lambda \hat{\vw}_0 + \sum_{i=0}^{s-1} \frac{1}{\epsilon_i^2} \left( \hat{\vw}^\top_0 \nabla f_{\vx_i}(\hat{\vw}_0) + y_i - f_{\vx_i}(\hat{\vw}_0)\right)\nabla f_{\vx_i}(\hat{\vw}_0)\right).
\end{align*}
Further, $\hat{\vw}_s - \vw^*$ can be written as
\begin{align}
\hat{\vw}_s - \vw^* &= \Sigma^{-1}_s \left(\sum_{i=0}^{s-1} \frac{1}{\epsilon_i^2} \nabla f_{\vx_i}(\hat{\vw}_0) \left( \nabla f_{\vx_i}(\hat{\vw}_0)^\top \hat{\vw}_0 + y_i - f_{\vx_i}(\hat{\vw}_0)\right) \right) + \lambda \Sigma^{-1}_s \hat{\vw}_0 - \Sigma^{-1}_s \Sigma_s \vw^*, \nonumber \\
&= \Sigma^{-1}_s \left(\sum_{i=0}^{s-1} \frac{1}{\epsilon_i^2} \nabla f_{\vx_i}(\hat{\vw}_0) \left( \nabla f_{\vx_i}(\hat{\vw}_0)^\top \hat{\vw}_0 + y_i - f_{\vx_i}(\hat{\vw}_0)\right) \right) + \lambda \Sigma^{-1}_s (\hat{\vw}_0 - \vw^*) \nonumber\\
&\qquad - \Sigma^{-1}_s \left( \sum_{i=0}^{s-1}  \frac{1}{\epsilon_i^2} \nabla f_{\vx_i}(\hat{\vw}_0) \nabla f_{\vx_i}(\hat{\vw}_0)^\top\right) \vw^*, \nonumber \\
&= \Sigma^{-1}_s \left(\sum_{i=0}^{s-1} \frac{1}{\epsilon_i^2} \nabla f_{\vx_0}(\hat{\vw}_0) \left( \nabla f_{\vx_i}(\hat{\vw}_0)^\top (\hat{\vw}_0 - \vw^*) + f_{\vx_i}(\vw^*) - f_{\vx_i}(\vw^*) + y_i - f_{\vx_i}(\hat{\vw}_0)\right) \right) \nonumber\\
&\qquad + \lambda \Sigma^{-1}_s (\hat{\vw}_0- \vw^*), \nonumber \\
&= \Sigma^{-1}_s \left(\sum_{i=0}^{s-1} \frac{1}{\epsilon_i^2} \nabla f_{\vx_i}(\hat{\vw}_0) (\nabla f_{\vx_i}(\hat{\vw_0})^\top (\hat{\vw_0}-\vw^*) + f_{\vx_i}(\vw^*) - f_{\vx_i}(\hat{\vw_0})) \right) \nonumber\\
& \qquad + \Sigma^{-1}_s \left(\sum_{i=0}^{s-1} \frac{1}{\epsilon^2_i} \nabla f_{\vx_i}(\hat{\vw}_0) (y_i - f_{\vx_i}(\vw^*))\right) +\lambda \Sigma^{-1}_s (\hat{\vw}_0 - \vw^*),\nonumber \\
&= \Sigma^{-1}_s \left(\sum_{i=0}^{s-1} \frac{1}{\epsilon_i^2} \nabla f_{\vx_i}(\hat{\vw}_0) \frac{1}{2} \|\vw^* - \hat{\vw}_0\|^2_{\nabla^2 f_{\vx_i}(\tilde{\vw})}\right)\nonumber \\
&\qquad + \Sigma^{-1}_s \left(\sum_{i=0}^{s-1} \frac{1}{\epsilon^2_i} \nabla f_{\vx_i}(\hat{\vw}_0) (y_i - f_{\vx_i}(\vw^*))\right) +\lambda \Sigma^{-1}_s (\hat{\vw}_0 - \vw^*),\label{eq:w_t_w_star}
\end{align}
where the second line is again by our choice of $\Sigma_s$ and the last equation is by the second order Taylor's theorem of $f_{\vx_i}(\vw^*)$ at $\hat{\vw}_0$ where $\tilde{\vw}$ lies between $\vw^*$ and $\hat{\vw}_0$. 

Multiply both sides of Eq. \eqref{eq:w_t_w_star} by $\Sigma^\frac{1}{2}_s$ and we have
\begin{align*}
\Sigma^\frac{1}{2}_s (\hat{\vw}_s - \vw^*) &= \frac{1}{2}\Sigma^{-\frac{1}{2}}_s \left(\sum_{i=0}^{s-1} \frac{1}{\epsilon_i^2}  \nabla f_{\vx_i}(\hat{\vw}_0)\|\vw^* - \hat{\vw}_0\|^2_{\nabla^2 f_{\vx_i}(\tilde{\vw})}\right) \\
&\qquad + \Sigma^{-\frac{1}{2}}_s \left(\sum_{i=0}^{s-1} \frac{1}{\epsilon^2_i}  \nabla f_{\vx_i}(\hat{\vw}_0) (y_i - f_{\vx_i}(\vw^*))\right) + \lambda \Sigma^{-\frac{1}{2}}_s (\hat{\vw}_0 - \vw^*).
\end{align*}
Take square of both sides and by inequality $(a+b+c)^2\leq 3(a^2 + b^2 + c^2)$ and we obtain
\begin{align}
\|\hat{\vw}_s - \vw^*\|^2_{\Sigma_s} &\leq \underbrace{3\left\| \sum_{i=0}^{s-1} \frac{1}{\epsilon^2_i}  \nabla f_{\vx_i}(\hat{\vw}_0) (y_i - f_{\vx_i}(\vw^*))\right\|^2_{\Sigma_s^{-1}}}_{(a)} + \underbrace{ 3\lambda^2 \|\hat{\vw}_0 - \vw^*\|^2_{\Sigma^{-1}_s} }_{(b)} \nonumber\\
&\qquad + \underbrace{\frac{3}{4}\left\|\sum_{i=0}^{s-1} \frac{1}{\epsilon_i^2}  \nabla f_{\vx_i}(\hat{\vw}_0)\|\vw^* - \hat{\vw}_0\|^2_{\nabla^2 f_{\vx_i}(\tilde{\vw})}\right\|^2_{\Sigma^{-1}_s}}_{(c)}.
\end{align}
Now we bound terms (a), (b), (c) separately. 

\textbf{Bounding (a).} Let
\begin{align*}
    \matE_s &= \diag(1/\epsilon^2_0, \cdots, 1/\epsilon^2_{s-1}),\\
    \matG_s &= [\nabla f_{\vx_0}(\hat{\vw}_0), \cdots, \nabla f_{\vx_{s-1}}(\hat{\vw}_0)]^\top,\\
    \vf_s &= [f_{\vx_0}(\vw^*), \cdots, f_{\vx_{s-1}}(\vw^*)],\\
    \vy_s &= [y_0, \cdots, y_{s-1}],
\end{align*}
then (a) can be rewritten as:
\begin{align*}
(a) = 3\left\| \matG_s^\top \matE_s (\vf_s - \vy_s) \right\|^2_{\Sigma_s^{-1}}.
\end{align*}
Define $\Gamma_s = \matE^\frac{1}{2}_s (\vy_s - \vf_s) \in \mathbb{R}^s$. Then
\begin{align*}
(a) &= 3\| \matG_s^\top \matE^\frac{1}{2}_s \Gamma_s\|^2_{\Sigma^{-1}_s},\\
&= 3 \Gamma^\top_s \matE^\frac{1}{2}_s \matG_s \Sigma^{-1}_s \matG^\top_s \matE^\frac{1}{2}_s \Gamma_s,\\
&\leq 3\|\Gamma_s\|^2_2 \cdot \|\matE^\frac{1}{2}_s \matG_x \Sigma^{-1}_s \matG^\top_s \matE^\frac{1}{2}_s\|_2,\\
&\leq 3 s \cdot \|\Gamma_s\|^2_{\infty} \cdot \tr(\matE^\frac{1}{2}_s \matG_x \Sigma^{-1}_s \matG^\top_s \matE^\frac{1}{2}_s),\\
&\leq 3 s \cdot \tr (\Sigma^{-1}_s \matG^\top_s \matE_s \matG_s),\\
&= 3 s \cdot \tr(\matI_{d_w} - \lambda \Sigma^{-1}_s),\\
&\leq 3 d_w s,
\end{align*}
where the first inequality is due to Cauchy-Schwarz inequality, the second inequality is by property of $\ell_2$ norm, and the third inequality is due to $|f_{\vx_i}(\vw^*) - y_i|\leq \epsilon_i, \forall i \in [s]$ (Lemma \ref{lem:qme}), and definition of $\Sigma_s = \lambda \matI + \matG_s \matE_s \matG_s^\top$.

\textbf{Bounding (b).} Then (b) can be bounded as:
\begin{align*}
(b) \leq 3\lambda^2 \|\Sigma^{-1}_s\|_\mathrm{op} \|\hat{\vw}-\vw^*\|^2_2 \leq 3\lambda \|\hat{\vw}-\vw^*\|^2_2 \leq \frac{3\lambda C^2_0}{T_0^2},
\end{align*}
where the first inequality is by Holder's inequality, the second inequality is due to property of $\Sigma_s$ and the last inequality follows from Eq. \eqref{eq:w0}.

\textbf{Bounding (c).} Next, (c) can be rewritten and bounded as
\begin{align}
(c) & \leq \frac{3}{4} \left\|\sum_{i=0}^{s-1} \frac{1}{\epsilon_i^2}  \nabla f_{\vx_i}(\hat{\vw}_0) \|\nabla^2 f_{\vx_i}(\tilde{\vw})\|_\mathrm{op} \|\vw^* - \hat{\vw}_0\|^2_2\right\|^2_{\Sigma^{-1}_s},\nonumber \\
&\leq \frac{3}{4} \left\|\sum_{i=0}^{s-1} \frac{C_h C^2_0}{\epsilon_i^2 T^2_0}  \nabla f_{\vx_i}(\hat{\vw}_0)\right\|^2_{\Sigma^{-1}_s},\nonumber \\
&= \frac{3C^2_h C_0^2}{4 T_0^4} \left\|\sum_{i=0}^{s-1} \frac{1}{\epsilon^2_i}  \nabla f_{\vx_i}(\hat{\vw}_0)\right\|^2_{\Sigma^{-1}_s},\label{eq:norm2}
\end{align}
where the first inequality is due to Holder's inequality, the second inequality is due to Assumption \ref{asm:bounded} and Eq. \eqref{eq:w0}. Further we can rewrite Eq. \eqref{eq:norm2} as
\begin{align*}
\frac{3C^2_h C_0^2}{4 T_0^4} \left\|\sum_{i=0}^{s-1} \frac{1}{\epsilon^2_i}  \nabla f_{\vx_i}(\hat{\vw}_0)\right\|^2_{\Sigma^{-1}_s} &= \frac{3C^2_h C_0^2}{4 T_0^4} \left( \sum_{i=0}^{s-1} \frac{1}{\epsilon^2_i}  \nabla f_{\vx_i}(\hat{\vw}_0)\right)^\top \Sigma^{-1}_s \left( \sum_{j=0}^{s-1} \frac{1}{\epsilon^2_j}  \nabla f_{\vx_j}(\hat{\vw}_0)\right),\\
&= \frac{3C^2_h C_0^2}{4 T_0^4} \sum_{i=0}^{s-1} \sum_{j=0}^{s-1} \frac{1}{ \epsilon^2_i \epsilon^2_j} \nabla f_{\vx_i}(\hat{\vw}_0)^\top \Sigma_s^{-1} \nabla f_{\vx_j}(\hat{\vw}_0),\\
&\leq \frac{3C^2_h C_0^2}{4 T_0^4} \sum_{i=0}^{s-1} \sum_{j=0}^{s-1} \frac{1}{ \epsilon^2_i \epsilon^2_j} \|\nabla f_{\vx_i}(\hat{\vw}_0)\|_{\Sigma_s^{-1}} \|\nabla f_{\vx_j}(\hat{\vw}_0)\|_{\Sigma_s^{-1}},\\
&= \frac{3C^2_h C_0^2}{4 T_0^4} \left(\sum_{i=0}^{s-1} \frac{1}{\epsilon^2_i} \|\nabla f_{\vx_i}(\hat{\vw}_0)\|_{\Sigma_s^{-1}}\right) \left(\sum_{j=0}^{s-1} \frac{1}{\epsilon^2_j} \|\nabla f_{\vx_j}(\hat{\vw}_0)\|_{\Sigma_s^{-1}}\right),\\
&= \frac{3C^2_h C_0^2}{4 T_0^4} \left(\sum_{i=0}^{s-1} \frac{1}{\epsilon_i} \right) \left(\sum_{j=0}^{s-1} \frac{1}{\epsilon_j} \right),\\
&\leq \frac{3C^2_h C_0^2}{4 T_0^4} \left(\sum_{i=0}^{s-1} \frac{1}{\epsilon^2_i} \right) \left( \sum_{i=0}^{s-1} 1\right),\\
&\leq \frac{3C^2_h C_0^2 s}{4 T_0^4} \left(\sum_{i=0}^{s-1} \frac{1}{\epsilon^2_i} \right),\\
&\leq \frac{3C^2_h C_0^2 s T^2}{4 T_0^4},
\end{align*}
where the first and second inequalities are due to Cauchy-Schwarz inequality and the last inequality is by Lemma \ref{lem:sum_eps_squ}.

Combine (a), (b), and (c), and we have:
\begin{align}
\|\hat{\vw}_s - \vw^*\|^2_{\Sigma_s} &\leq 3d_w s + \frac{3\lambda C_0^2}{T_0^2} + \frac{3C^2_h C_0^2 s T^2}{4 T_0^4},
\end{align}
which reflects our choice of $\beta_s$.
\end{proof}

\section{Time Complexity Analysis}\label{app:time}
In this section, we show the time complexity of Q-NLB-UCB compared with other quantum bandit algorithms. The time complexity of Q-NLB-UCB is demonstrated to be lower than that of other quantum non-linear bandit algorithms in terms of $T$.

\begin{table}[!htbp]
\centering
\begin{tabular}{lcc}
\toprule
Algorithm & Quantum part & Classical part \\
\midrule
Q-GP-UCB   & $O(TU_{0})$ & $O(d_{x}^{4}(\log T)^{4})$ \\
QMCKernelUCB  & $O(TU_{0})$ & $O(d_{x}^{4}(\log T)^{4})$ \\
Q‑NLB‑UCB (ours)    & $O(TU_{0})$ & $O(d_{w}^{4}\log T)$ \\
\bottomrule
\end{tabular}
\caption{Time complexity comparison}
\label{tab:1}
\end{table}

We show the time complexity of Q-NLB-UCB in Table~\ref{tab:1} in comparison with Q-GP-UCB \citep{DLV+23} and QMCKernelUCB \citep{hikimaquantum}. The time complexity is analyzed in their quantum parts and classical parts, respectively. Let $U_{0}$ denote the cost for one query of quantum oracle, then all algorithms have quantum time complexity $O(T U_0)$ since they all run in $T$ rounds.

However, in the classical part, we can see that our Q‑NLB‑UCB algorithm enjoys a better time complexity in terms of $T$. Note $d_{w}$ denotes parameter dimension and $d_{x}$ denotes input dimension. In \Cref{alg:q_go_ucb}, line 3,4,5,6 and 7 are all classical steps, and their time complexity is dominated by inversion of the covariance matrix $\Sigma_{s}\in \matR^{d_{w}\times d_{w}}$, which requires time complexity $O(d_{w}^{3})$. Therefore, total classical time complexity of Q-NLB-UCB becomes $O(m d_{w}^{3}) = O(d_{w}^{4}\log T)$. While in Q-GP-UCB and QMCKernelUCB, the classical steps are dominated by the kernel matrix inversion, which takes $O(s^{3})$, so the time complexity becomes $O(\sum_{s=1}^{m}s^{3}) = O(m^{4}) = O(d_{x}^{4}(\log T)^{4})$ in total.

\section{Details of Experiments}\label{app:exp}

\subsection{Implementation Details}

We publicly release our source code on GitHub \url{https://github.com/ZakSiam/Quantum-Non-Linear-Bandit-Optimization} to maximize the reproducibility of our work. We also provide full implementation details in this supplement.

\subsubsection{For Algorithms}

An artificial noise variance \( \sigma^2 = 0.01 \) (aligning with our amplitude-estimation setup) was incorporated when interfacing with the quantum subroutine. The regression oracle in Q-NLB-UCB is approximated by the stochastic gradient descent algorithm on our two layer neural network model utilizing the mean squared error loss, $2000$ iterations and $10^{-3}$ learning rate. Cross-optimization problem in Step 6 of Q-NLB-UCB is approximated by the iterative gradient ascent algorithm over $\vx$ and $\vw$ with 2000 iterations. We set the learning rate of $\vx$ and $\vw$ to $10^{-3}$ for the synthetic experiments, and $10^{-4}$ for all the AutoML experiments. We set $\lambda = T$, $\beta_s = C\,\log\bigl(s+1.0\bigr)$ with $C=1.0$, $C_g=18.0$, $C_1=1.0$, and the fail probability $\delta=0.01$.

We compute the runtime of our Q-NLB-UCB algorithm by recording the wall‐clock time (using \texttt{time.perf\_counter} method in Python) that elapses from just before the main procedure begins until immediately after it terminates. Additionally, we compare the runtime of Q-NLB-UCB with the runtime of Q-GP-UCB and QMCKernelUCB. This runtime encompasses all overheads, such as, circuit construction, data management, amplitude estimation calls etc., providing a fair end‐to‐end measure for each algorithm under identical conditions while performing the comparative analysis of runtime. In order to ensure consistent hardware and software settings, we run each script separately in the same environment. Finally, the total reported runtime for each algorithm is utilized for direct comparisons.

All experiments were executed locally on a MacBook Pro machine which is equipped with an Apple M4 Pro system-on-chip featuring a 12-core CPU, a 16-core integrated GPU, and a 16-core Neural Engine. It has 24 GB of unified memory shared by the CPU / GPU, and a 512 GB SSD for storage. Code was run under macOS 15 with Python 3.10, PyTorch 2.1, and Qiskit 0.42.1; no external accelerators were used.

\subsubsection{For Real-World AutoML Tasks}

In our AutoML experiments, we delineate that the Q-NLB-UCB algorithm works efficiently in real-world tasks. We particularly focus on three different hyperparameter tuning tasks for three classifiers on two different datasets: Pima Indians Diabetes Database and Breast Cancer Wisconsin (Diagnostic) Dataset. We adopt the Diabetes dataset from Kaggle, and the Breast Cancer dataset from \texttt{sklearn.datasets.load\_breast\_cancer}. Below is the information on dataset licenses.

\begin{itemize}[leftmargin=1.2em]
  \item \textbf{Pima Indians Diabetes Database}  \footnote{https://www.kaggle.com/datasets/uciml/pima-indians-diabetes-database},  
        released under \textit{CC0: Public Domain}.  
  \item \textbf{Breast Cancer Wisconsin (Diagnostic) Dataset}, bundled with \texttt{scikit-learn 1.4.2},  
        original UCI source © Wolberg et al., re-licensed under \textit{CC BY-NC-SA 4.0}.  
\end{itemize}

The diabetes dataset features eight numeric predictors (glucose level, body mass index, number of times pregnant, age, blood pressure, insulin, skin thickness, and diabetes pedigree function) and a binary outcome denoting diabetes or not (1 or 0). On the other hand, the Breast Cancer Wisconsin (Diagnostic) Dataset features 30 real-valued predictive features and a binary outcome (benign or malignant). For each dataset and experiment, we employed the 5-fold cross validation technique by dividing each dataset into 5 folds and then every time using 4 folds for training and remaining 1 fold for testing, in order to reduce the effect of randomness.

We tuned the following eight hyperparameters for the Multi-Layer Perceptron (MLP):

\begin{itemize}
    \item \textbf{Activation function}: The activation function used in the hidden layer. Options include \texttt{"identity"}, \texttt{"logistic"}, \texttt{"tanh"}, or \texttt{"relu"}.
    
    \item \textbf{L2 regularization term}: Strength of the L2 regularization, which helps prevent overfitting. It is a float value sampled from $[10^{-6}, 10^{-2}]$.
    
    \item \textbf{Initial learning rate}: Initial learning rate used for weight updates. Float in the range $[10^{-6}, 10^{-2}]$.
    
    \item \textbf{Maximum iterations}: The maximum number of training iterations. Integer value between 100 and 300.
    
    \item \textbf{Shuffle}: Whether to shuffle samples in each iteration. Boolean value (\texttt{True} or \texttt{False}).
    
    \item \textbf{$\beta_1$}: Exponential decay rate for estimates of the first moment vector. Float in the open interval $(0, 1)$.
    
    \item \textbf{$\beta_2$}: Exponential decay rate for estimates of the second moment vector. Float in the open interval $(0, 1)$.
    
    \item \textbf{Max epochs without improvement}: The maximum number of epochs to not meet the tolerance improvement. Integer between 1 and 10.
\end{itemize}

We tuned the following eleven hyperparameters for Gradient Boosting (GB):

\begin{itemize}
    \item \textbf{Loss}: The loss function to be optimized. Options include \texttt{"log\_loss"} or \texttt{"exponential"}.
    
    \item \textbf{Learning rate}: A float in the open interval $(0, 1)$ controlling the contribution of each tree.
    
    \item \textbf{Number of estimators}: The total number of boosting stages. Integer in $[20, 200]$.
    
    \item \textbf{Subsample fraction}: The fraction of samples to be used for fitting the individual base learners. Float in the open interval $(0, 1)$.
    
    \item \textbf{Split quality criterion}: The function to measure the quality of a split. Options include \texttt{"friedman\_mse"} or \texttt{"squared\_error"}.
    
    \item \textbf{Min samples to split internal node}: Minimum number of samples required to split an internal node. Integer in $[2, 10]$.
    
    \item \textbf{Min samples at leaf node}: Minimum number of samples required to be at a leaf node. Integer in $[1, 10]$.
    
    \item \textbf{Min weight fraction}: Minimum weighted fraction of the sum total of weights required to be at a leaf node. Float in the interval $(0, 0.5]$.
    
    \item \textbf{Max depth}: Maximum depth of the individual regression estimators. Integer in $[1, 10]$.
    
    \item \textbf{Number of features}: Number of features to consider when looking for the best split. Options include \texttt{"sqrt"}, or \texttt{"log2"}.
    
    \item \textbf{Max leaf nodes}: Maximum number of leaf nodes in best-first fashion. Integer in $[2, 10]$.
\end{itemize}

We tuned the following four hyperparameters for support vector machine (SVM):
\begin{itemize}
    \item \textbf{C}: The regularization parameter which controls the trade-off between having a smooth decision boundary (smaller \texttt{C}) versus classifying training points accurately (larger \texttt{C}). 
    \item \textbf{gamma ($\gamma$)}: The kernel coefficient (for \texttt{rbf} kernels), which impacts the extent to which a single training sample affects the decision boundary. A higher \texttt{gamma} value generally tends to create more complex boundaries.
    \item \textbf{tol} (\texttt{tolerance}): The stopping tolerance for the solver, specifically the minimal variation in objective value that the solver keeps iterating upon. Lower values result in a more accurate but sometimes slower convergence.
    \item \textbf{kernel}: The functional form which is utilized to map feature vectors into a higher-dimensional space. Common choices for \texttt{SVC} include \texttt{linear}, \texttt{rbf}, \texttt{poly}, and \texttt{sigmoid}, each of which can significantly change both model capacity and performance. We only chose \texttt{linear} and \texttt{rbf} kernels for our experiments.
\end{itemize}

We employed scikit‐learn’s \texttt{svm.SVC}, \texttt{neural\_network.MLPClassifier}, and\linebreak[4] \texttt{ensemble.GradientBoostingClassifier} to train and evaluate the SVM, MLP, and Gradient-Boosting models, respectively, while systematically varying their key hyperparameters through our continuous search embedding. For each learner—SVM (4D), Multi-Layer Perceptron (8D) and Gradient Boosting (11D) — we optimize over a continuous box $\cX=[0,10]^d$ sampled by a scrambled Sobol sequence. Every evaluation point $x\in\cX$ is decoded into concrete scikit-learn hyperparameters by simple affine or logarithmic rescaling for numeric fields and by partitioning the interval for categorical choices (e.g.\ $x_4<5$ selects a \texttt{linear} SVM kernel, $x_1<2.5$ picks the \texttt{identity} activation in MLP, $x_1<5$ chooses the ``log\_loss'' objective in Gradient Boosting). This continuous embedding allows gradient updates while still covering the usual discrete options. The objective returned to the bandit algorithm is the mean test accuracy across the same five pre-generated stratified folds of the Breast-Cancer dataset; all classifiers use \texttt{random\_state=0} so that variability arises solely from the Sobol-initialized search trajectory.  

The primary objective here is to maximize the validation accuracy in both datasets by choosing the optimal set of hyperparameter values. Therefore, the function mapping from hyperparameters to classification accuracy is the black-box function that we aim to maximize for each classifier on each dataset. Cumulative regret was used to evaluate hyperparameter tuning performances, however, since the best accuracy $f^*$ is unknown ahead of time, therefore, it was set to be the best empirical accuracy of each task.

\subsection{Additional Experimental Results}

\subsubsection{Real-World AutoML Results}

\begin{figure}[!htbp]
    \centering
    \begin{minipage}{0.24\linewidth}\centering
		\includegraphics[width=\textwidth]{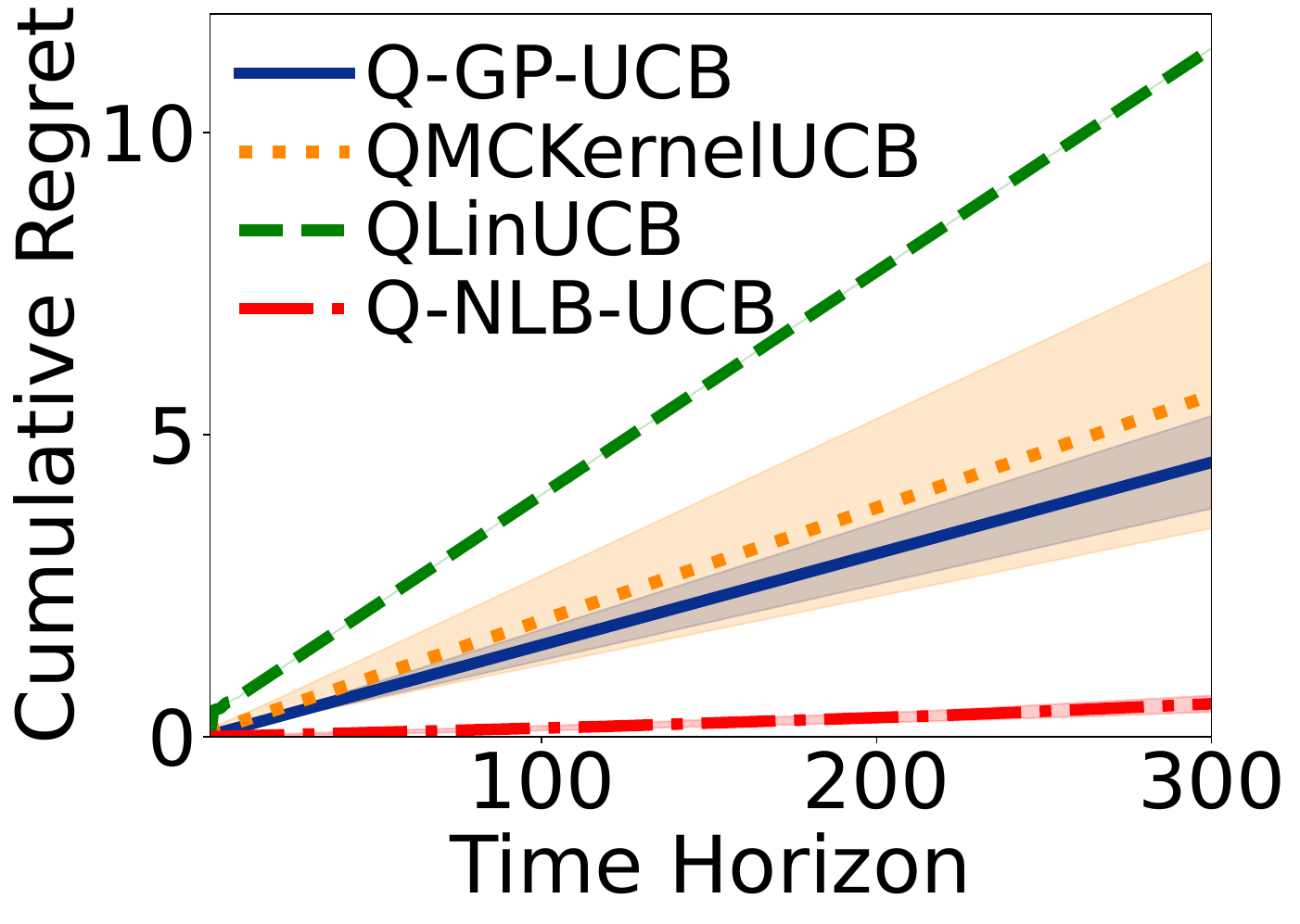}
        \footnotesize
		(a) 11D GB AutoML (Cancer)
	\end{minipage}
    \begin{minipage}{0.24\linewidth}\centering
		\includegraphics[width=\textwidth]{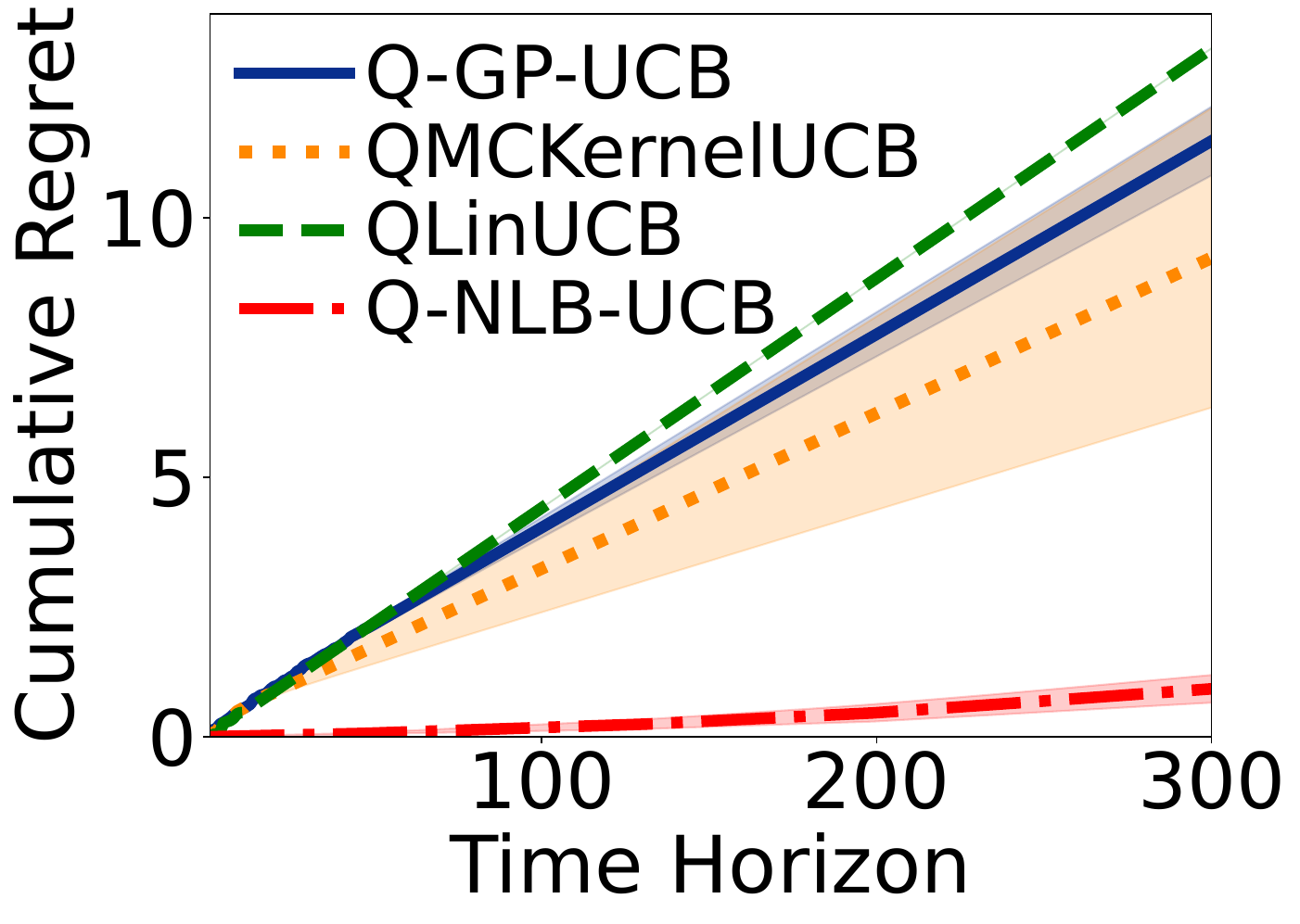}
        \footnotesize
		(b) 11D GB AutoML (Diabetes)
	\end{minipage}
    \begin{minipage}{0.24\linewidth}\centering
		\includegraphics[width=\textwidth]{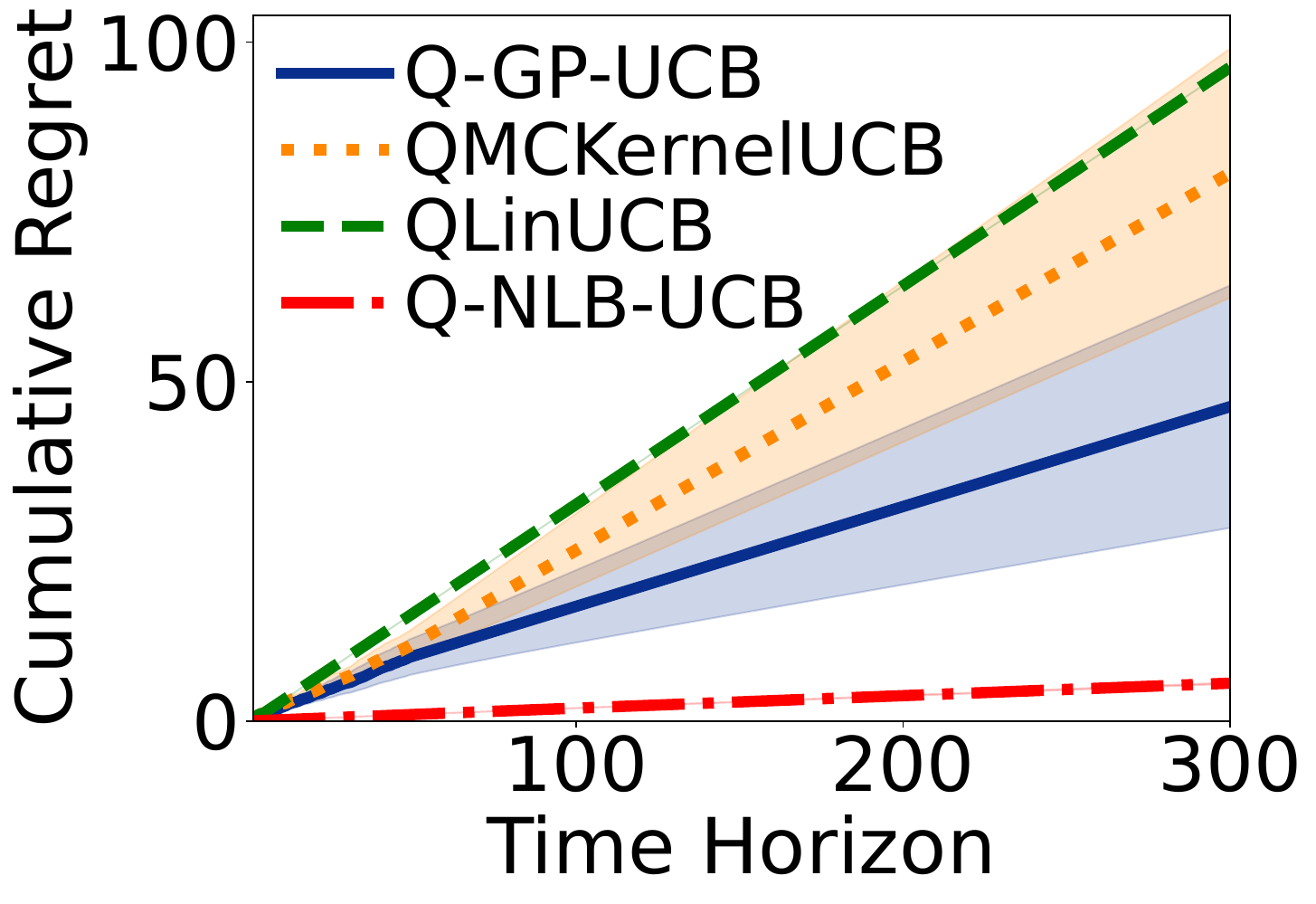}
        \footnotesize
		(c) 4D SVM AutoML (Cancer)
	\end{minipage}
    \begin{minipage}{0.24\linewidth}\centering
		 \includegraphics[width=\textwidth]{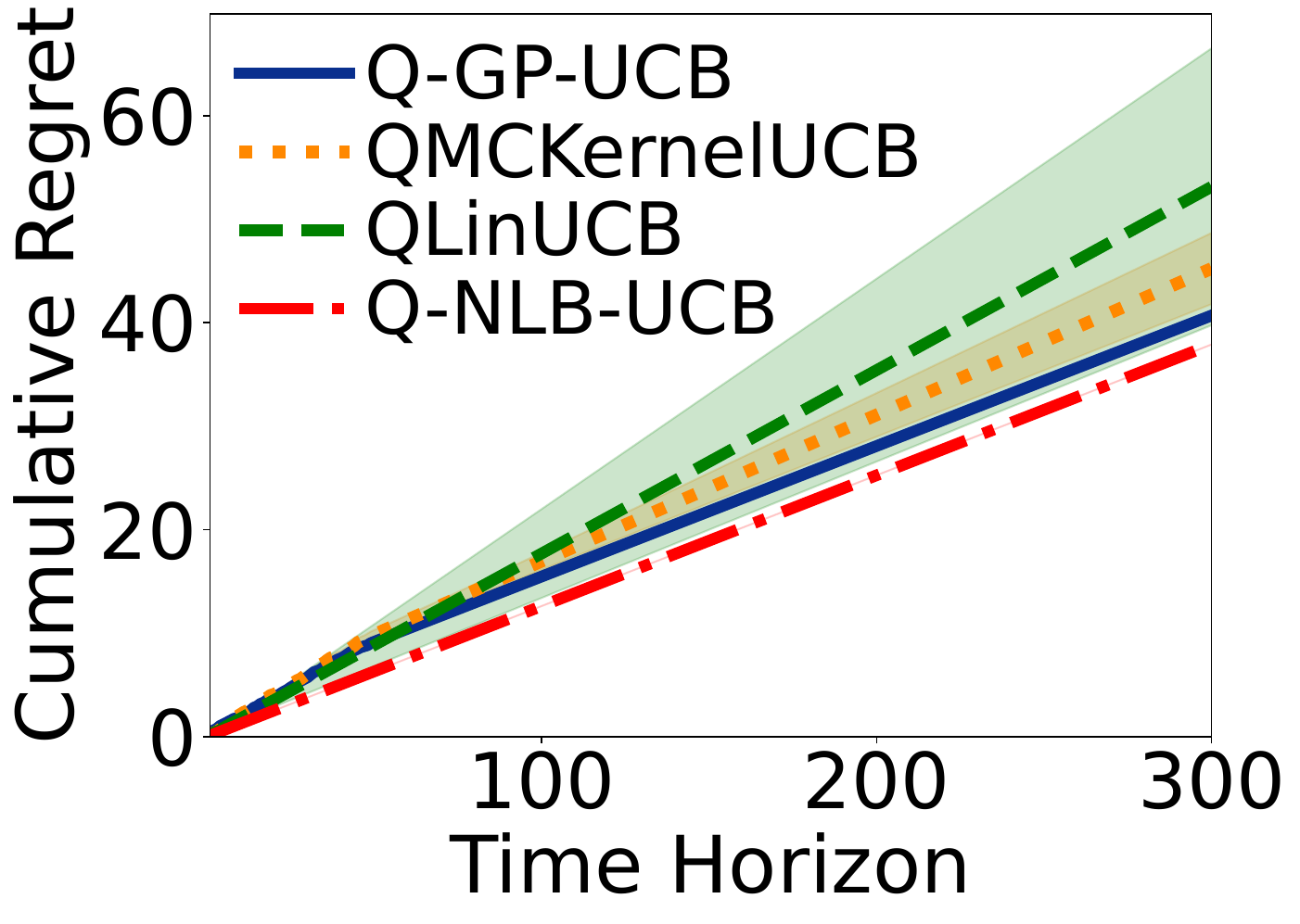}
         \footnotesize
		(d) 4D SVM AutoML (Diabetes)
	\end{minipage}
    \caption{Cumulative regrets (the lower the better) of all compared quantum bandit algorithms.}
    \label{fig:automl_gb_svm}
\end{figure}

\begin{table}[!htbp]
\centering
\begin{tabular}{lcccc}
\noalign{\smallskip}
\toprule
\noalign{\smallskip}
\textbf{Stage} & \textbf{Selected action $x_s$} & \textbf{True reward $f(x_s)$} & \textbf{QME estimated $y_s$} & \textbf{$|y_s - f(x_s)|$} \\
\noalign{\smallskip}
\midrule
\noalign{\smallskip}
1 & $[-0.6, 0.5, -0.1]$ & -40.6200 & -40.6109 & 0.0091 \\
2 & $[-0.5, 0.4, 0]$    & -38.5002 & -38.5085 & 0.0083 \\
3 & $[-0.5, 0.3, 0]$    & -33.4302 & -33.4329 & 0.0027 \\
4 & $[-0.5, 0.3, 0.1]$  & -35.3500 & -35.3491 & 0.0009 \\
5 & $[-0.5, 0.3, 0.1]$  & -35.3500 & -35.3486 & 0.0014 \\
\noalign{\smallskip}
\bottomrule
\end{tabular}
\caption{Quantum Implementation Results}
\label{tab:quantum}
\end{table}

\begin{table}
\centering
\begin{tabular}{lcccc}
\noalign{\smallskip}
\toprule
\noalign{\smallskip}
\textbf{Stage} & \textbf{Selected action $x_s$} & \textbf{True reward $f(x_s)$} & \textbf{Classical mean $y_s$} & \textbf{$|y_s - f(x_s)|$} \\
\noalign{\smallskip}
\midrule
\noalign{\smallskip}
1 & $[-0.6, 0.5, -0.1]$ & -40.6200 & -26.1172 & 14.5028 \\
2 & $[-0.7, 0.5, -0.1]$ & -35.7500 & -33.2826 & 2.4674 \\
3 & $[-0.6, 0.5, -0.1]$ & -40.6200 & -40.7809 & 0.1609 \\
4 & $[-0.5, 0.4, 0]$    & -38.5002 & -36.8721 & 1.6281 \\
5 & $[-0.5, 0.4, 0]$    & -38.5002 & -40.9617 & 2.4615 \\
\noalign{\smallskip}
\bottomrule
\end{tabular}
\caption{Classical Implementation Results}
\label{tab:classical}
\end{table}

In the main paper, we only show results on the MLP hyperparameter tuning tasks. Here in Figure~\ref{fig:automl_gb_svm}, we show results on GB and SVM tasks. It again shows that our Q-NLB-UCB algorithm outperforms all the other algorithms by achieving significantly smaller cumulative regret on both datasets.

Our experiments reveal that Q-NLB-UCB can consistently find out the near-optimal Gradient Boosting and SVM hyperparameters while utilizing relatively few queries. The quantum amplitude estimation step in Q-NLB-UCB allows it to refine its hyperparameter selection in a way that balances exploration and exploitation very effectively. The final classification accuracies at the identified hyperparameters were very close to the true maximum of each dataset’s domain, delineating the viability of Q-NLB-UCB as a quantum-inspired powerful tool for hyperparameter tuning on different real-world datasets.

\subsubsection{Ablation Experiments on Quality of Quantum Mean Estimator}

For Step 1 of Q-NLB-UCB in \Cref{alg:q_go_ucb}, while many algorithms have been designed to solve quantum linear regression problems, quantum non-linear regression is still an open problem where no specific algorithm has been developed to achieve the quantum speed-up. Our work in Step 1 proves the existence of a quantum regression oracle that solves this problem, shading the light to develop such a realizable algorithm in the future. Therefore, in experiments, we used classical gradient descent as the surrogate model for Step 1. Even with that, our Q-NLB-UCB algorithm outperforms all other compared quantum bandit algorithms.

For Step 9 in \Cref{alg:q_go_ucb}, we have conducted the ablation experiments, shown in Table~\ref{tab:quantum} and Table~\ref{tab:classical}. We can see that observations from quantum oracles are much closer to the true function value than observations from the classical query, showing great quality of QME.

\end{document}